\documentclass[11pt,a4paper,twocolumn]{article}

\usepackage[a4paper,top=1in,bottom=1in,left=1in,right=1in]{geometry}
\usepackage{times}
\usepackage[T1]{fontenc}
\usepackage[utf8]{inputenc}
\usepackage{microtype}
\usepackage{amsmath,amssymb,amsthm}
\usepackage{graphicx}
\usepackage{booktabs}
\usepackage{algorithm}
\usepackage{algpseudocode}
\usepackage[round]{natbib}
\usepackage{caption}
\usepackage{url}
\usepackage[colorlinks=true,linkcolor=blue,citecolor=blue,urlcolor=blue]{hyperref}
\usepackage{titlesec}
\titleformat{\section}{\large\bfseries}{\thesection}{1em}{}
\titleformat{\subsection}{\normalsize\bfseries}{\thesubsection}{1em}{}
\titlespacing*{\section}{0pt}{1.6ex plus .6ex minus .2ex}{0.9ex plus .2ex}
\titlespacing*{\subsection}{0pt}{1.3ex plus .5ex minus .2ex}{0.7ex plus .2ex}
\graphicspath{{figures/}}

\newtheorem{proposition}{Proposition}
\newtheorem{corollary}{Corollary}
\newtheorem{definition}{Definition}

\newcommand{\E}{\mathbb{E}}
\newcommand{\Corr}{\mathrm{Corr}}
\newcommand{\bfpara}[1]{\noindent\textbf{#1}}

\title{\textbf{Proof-Carrying Cognition:\\
Closing the Verification Gap with Reality-Settled Reward}}

\author{%
  \begin{tabular}{c}
    \textbf{Eshwar Reddy M} \\
    AI Engineer, Testsigma \\
    University of San Diego \\
    \texttt{malireddy.eshwar@gmail.com}
  \end{tabular}
  \hspace{2.5em}
  \begin{tabular}{c}
    \textbf{Sourav Karmakar} \\
    Senior AI Scientist, Intuit India \\
    \\
    \texttt{souravkarmakar29@gmail.com}
  \end{tabular}
}
\date{}

\begin{document}
\maketitle

\begin{abstract}
Recent frontier progress in large language models has come predominantly from reinforcement learning on reasoning traces---and has been sharply concentrated in domains that possess a cheap, sound verifier. We argue that the field's binding constraint is therefore the \emph{verification gap}: the absence of a scalable, incorruptible source of reward for reasoning outside narrowly formal domains. We make four contributions. \textbf{(1) Theory:} in a joint-Gaussian model of best-of-$N$ selection we prove that verifier--gold correlation $\rho$ is the exact exchange rate between test-time compute and capability, and that an unsound verifier must pay a polynomial compute penalty $N^{1/\rho^2}$ to match a sound one. \textbf{(2) Demonstration:} in program-synthesis testbeds with an executable ground truth---a minimal six-token domain and a pre-registered scaled replication in a $\sim\!10^{10}$-program domain with 2--4$\times$ more tasks and stronger instruments---we show that unsound verifiers lose Soundness-under-Pressure as optimization pressure grows (from 0.94 to 0.32 for a shallow verifier at $N{=}4096$), collapsing absolutely in the small domain and plateauing in the rich one, while a sound verifier improves monotonically; that reality-anchored settlement beats a frozen verifier under both i.i.d.\ and directed adversarial pressure, including against an adversary that explicitly models the settlement process, driving the hacking gap from $\sim\!0.27$ to $\sim\!0$; that soundness scales log-linearly with settled labels, with on-policy settlement $\sim\!10\times$ more label-efficient than random labeling and decisively better than an uncertainty-sampling control; and that settlement improves the \emph{resolution} of claims, not merely their calibration. The replication also falsified two of our quantitative claims, which we report and repair: outside the Gaussian model the exchange rate $\rho$ is not quantitatively predictive, and the closed form $N^{1/\rho^2}$ overstates the penalty at practical $N$ (the exact finite-$N$ form of Proposition~\ref{prop:exchange} predicts matching budgets to within 3.4\%). \textbf{(2b) Real-model evidence:} in pre-registered experiments with a real frontier-family generator, real LLM judges, and real unit-test execution as gold, the weak judge loses soundness under best-of-$N$ on both MBPP and HumanEval ($p<10^{-3}$); a stronger judge is significantly more robust---so the gap binds \emph{conditionally} on judge capability relative to task difficulty; a learned settlement model anchored to executed outcomes cuts the judge's pricing error by 59\%, while naive in-context anchoring makes it worse; and a prompt-level LLM adversary failed to inflate the judge---while mining the same candidate banks shows selection alone manufactures $+0.53$ hacking gaps from honest samples: the operative Goodhart pressure is selection, not persuasion. A pre-registered diagnosis of every failure yields a repaired, margin-free exchange-rate law (the copula form predicts realized soundness of real LLM judges to 4\% median error) and a design constraint for the paradigm: average-calibration anchoring provably cannot move soundness and empirically amplifies tail deceptions---the world model must match the expressiveness of the deception surface. \textbf{(3) Paradigm:} we propose \emph{proof-carrying cognition}, a training loop in which reasoning steps are emitted as typed probabilistic claims, priced by a self-built world model whose only loss is prediction of held-out reality, and settled by strictly proper scoring rules---making reality, rather than human judgment, the ultimate reward function. \textbf{(4) Benchmark:} we specify Soundness-under-Pressure, the headline metric for a reality-settled reasoning benchmark, and argue that building this benchmark is the field's single most important near-term action. \textbf{(2c)} Under real GRPO training, a frozen learned RM traces the full overoptimization curve---proxy reward climbs while executed reward collapses by 90\%---and the identical RM refit on a 10\% settlement stream preserves executed reward at 6$\times$ the frozen arm's and outperforms an equal-budget LLM-judge-updated control, isolating reality as the label source; the registered drift-alarm criterion failed its first test and is reported as such.
\end{abstract}

\section{Introduction}

Two inventions define modern machine learning: backpropagation, which solved credit assignment over parameters \citep{rumelhart1986learning}, and the transformer, which solved bandwidth over context \citep{vaswani2017attention}. We argue that the next invention of comparable consequence must solve \emph{credit assignment over thoughts}: determining, scalably and incorruptibly, whether a step of reasoning is correct.

The argument proceeds from an empirical observation. Frontier capability gains now come predominantly from reinforcement learning on chains of thought \citep{wei2022chain,openai2024learning,guo2025deepseek}, and these gains are sharply concentrated in domains with a cheap, sound verifier: mathematics with checkable answers \citep{hendrycks2021measuring}, code with executable tests \citep{chen2021evaluating}, and formal proof with a kernel \citep{trinh2024solving,alphaproof2024}. The underlying recipe---sample many traces, keep those a verifier accepts, reinforce \citep{zelikman2022star}---is a self-improvement engine whose throttle is verifier coverage. Where the verifier is sound, months of training have produced superhuman-adjacent performance. Where the only reward is a human preference label \citep{christiano2017deep,ouyang2022training} or an LLM judge, the same models plateau at fluent but unaccountable reasoning, because preference-based reward is noisy, gameable, and shallower than the reasoning it evaluates \citep{gao2023scaling}.

This paper names that constraint the \emph{verification gap} and attacks it from four directions:
\begin{itemize}\setlength{\itemsep}{2pt}
  \item \textbf{Theory (\S\ref{sec:theory}).} We prove that in best-of-$N$ selection, verifier--gold correlation $\rho$ is the exchange rate between compute and capability, with an unsound verifier paying a polynomial penalty $N^{1/\rho^2}$ in candidates (Proposition~\ref{prop:exchange}, Corollary~\ref{cor:penalty}); simulation matches theory to within $0.011\sigma$ (Figure~\ref{fig:theory}).
  \item \textbf{Demonstration (\S\ref{sec:demos}--\S\ref{sec:replication}).} In program-synthesis testbeds with an executable gold verifier: learned verifiers are Goodharted under best-of-$N$ pressure (relative soundness $0.94\to0.32$ at scale; absolute collapse in the minimal domain) while a sound verifier climbs monotonically; reality-anchored settlement beats a frozen verifier under directed adversarial pressure, including a settlement-aware adversary, driving the hacking gap to $\sim\!0$; verifier soundness scales log-linearly with settled labels, with on-policy settlement $\sim\!10\times$ more label-efficient than random labeling; and settlement raises claim \emph{resolution}, rebutting the concern that settlement-based reward favors vague claims. \S\ref{sec:replication} reports a pre-registered scaled replication, including the parts of our story it falsified.
  \item \textbf{Paradigm (\S\ref{sec:pcc}).} We propose \emph{proof-carrying cognition} (Algorithm~\ref{alg:pcc}): reasoning as portfolios of falsifiable, priced claims, with reality as the only unamortized loss.
  \item \textbf{Benchmark (\S\ref{sec:bench}).} We give a formal metric, Soundness-under-Pressure, and a construction recipe for a reality-settled reasoning benchmark.
\end{itemize}

\bfpara{Scope note:} this is a hybrid position-and-proof-of-concept paper. The experiments in \S\ref{sec:theory}--\S\ref{sec:demos} are real and reproducible but deliberately minimal; they validate mechanisms, not frontier-scale claims.

\section{The Bottleneck: Generation Has Scaled; Verification Has Not}

\subsection{The verification gap}

Define a \emph{verifier} for a domain $D$ as an oracle $V:\mathrm{claims}(D)\to[0,1]$ that is (a) \emph{sound}---its acceptances track ground truth; (b) \emph{cheap}---evaluating $V$ costs far less than generating a candidate; and (c) \emph{robust under optimization}---an adversarially trained generator cannot systematically obtain reward from $V$ for false claims. The verification gap is the observation that such oracles exist today only for formally grounded domains (proof kernels, program execution, game rules), and that every substitute used elsewhere---human raters, reward models, LLM judges---fails property (c) under sufficient optimization pressure. \citet{gao2023scaling} quantified the failure directly: optimizing against a fixed learned reward model yields proxy reward that rises while gold reward peaks and then falls---Goodhart's law \citep{strathern1997improving} operating as a hard ceiling. A model cannot climb above the discrimination ability of its judge \citep{leike2018scalable,bowman2022measuring}.

\begin{figure}[t]
  \centering
  \includegraphics[width=\columnwidth]{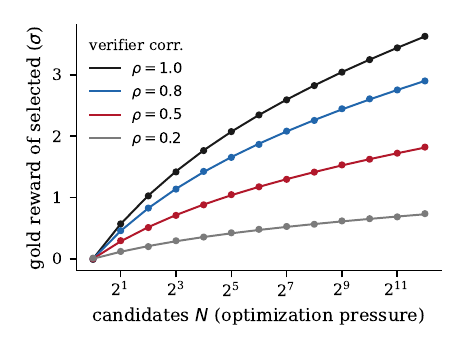}
  \caption{Simulation (dots; 20,000 trials per point) versus theory (lines; $\rho\,\E[\max_N]$) for best-of-$N$ selection under a proxy verifier with correlation $\rho$ to gold. Maximum absolute deviation from theory across all points: $0.011\sigma$. At $N{=}4096$, the $\rho{=}0.5$ verifier realizes 50.2\% of the sound verifier's gain, matching Proposition~\ref{prop:exchange} exactly.}
  \label{fig:theory}
\end{figure}

\subsection{Other bottlenecks reduce to it}

\bfpara{The data wall.} High-quality human text is finite, but self-generated data is unbounded---\emph{if} it can be trusted. AlphaGo Zero \citep{silver2017mastering} showed that with a perfect verifier (game rules), self-play alone yields superhuman skill with no human data. AlphaProof \citep{alphaproof2024} extended this to mathematics, reaching IMO silver-medal standard on millions of autoformalized problems adjudicated by the Lean kernel. The data wall is not a shortage of text; it is a shortage of trustworthy reward for synthetic reasoning.

\bfpara{Unreliable agency.} Long-horizon autonomy fails because errors compound with no mechanism to check intermediate steps before commitment. A sound step-level verifier converts open-loop generation into closed-loop, auditable action.

\bfpara{Superhuman reasoning in open domains.} In science, strategy, law, and medicine, humans can no longer reliably grade superhuman homework; scalable oversight \citep{amodei2016concrete,bowman2022measuring} is precisely the demand for verification that outruns the verifier's unaided competence.

\section{Theory: Verifier Correlation Is the Compute--Capability Exchange Rate}
\label{sec:theory}

We first make the cost of unsoundness quantitative in the cleanest possible model. Best-of-$N$ selection---sample $N$ candidates, keep the one the verifier scores highest---is the canonical mechanism by which test-time compute is converted into capability \citep{cobbe2021training,snell2024scaling}, and its induced optimization pressure grows as $\log N$.

\begin{proposition}[Verifier exchange rate]\label{prop:exchange}
Let $(G_i,P_i)_{i=1}^{N}$ be i.i.d.\ jointly Gaussian pairs, standardized, with $\Corr(G,P)=\rho$, where $G$ is gold reward and $P$ the proxy verifier's score. Let $i^*=\arg\max_i P_i$. Then
\[
\E[G_{i^*}] \;=\; \rho\,\E\!\left[\max_i P_i\right] \;=\; \rho\sqrt{2\ln N}\,\big(1+o(1)\big).
\]
\end{proposition}
\begin{proof}
Decompose $G_i=\rho P_i+\sqrt{1-\rho^2}\,Z_i$ with $Z_i$ standard Gaussian, independent of $(P_1,\dots,P_N)$. The index $i^*$ is a function of $(P_1,\dots,P_N)$ alone; conditioning on these, $Z_{i^*}$ remains standard Gaussian with mean zero, so $\E[Z_{i^*}]=0$ and $\E[G_{i^*}]=\rho\,\E[P_{i^*}]=\rho\,\E[\max_i P_i]$. The asymptotic $\E[\max_i P_i]\sim\sqrt{2\ln N}$ is standard extreme-value theory for Gaussians.
\end{proof}

\begin{corollary}[Polynomial compute penalty of unsoundness]\label{cor:penalty}
To match the expected gold reward that a sound verifier ($\rho{=}1$) attains with $N$ candidates, a proxy verifier with correlation $\rho$ requires
\[
N' \;=\; N^{1/\rho^2}\quad\text{(asymptotically).}
\]
\end{corollary}
\begin{proof}
Set $\rho\sqrt{2\ln N'}=\sqrt{2\ln N}$ and solve: $\ln N'=\ln N/\rho^2$.
\end{proof}

\bfpara{Two readings.} First, \emph{within this model class, verifier correlation is the exchange rate between test-time compute and capability}: every unit of inference compute is worth exactly $\rho$ of its sound-verifier value (Figure~\ref{fig:theory} and the 100,000-trial replication of \S\ref{sec:replication} confirm this to within $0.011\sigma$ and $0.007\sigma$ respectively). Outside the model class the identity is \emph{not quantitatively predictive in either direction}: in the synthetic testbed of \S\ref{sec:replication}, a verifier with measured $\rho\approx0.50$ realizes only $\approx0.32$ of the sound verifier's value at $N{=}4096$ (exploitable structure beyond Gaussian noise), while real LLM judges (\S\ref{sec:real}) with near-zero \emph{linear} correlation ($\rho\approx0.12$) realize $0.74$ at $N{=}32$ (good top-ranking despite poor linear fit). Only the qualitative ordering---higher-fidelity judges convert pressure to capability better---survived every testbed. The repair (\S\ref{sec:repair}) is to state the law in the invariant that selection actually sees:

\begin{proposition}[Margin-free exchange rate]\label{prop:copula}
Let $(G_i,P_i)_{i=1}^N$ be i.i.d.\ with continuous margins and Gaussian copula parameter $\rho_c$, and $i^*=\arg\max_i P_i$. Then the law of $G_{i^*}$---hence $\Snd@N$---depends only on $\rho_c$ and the margin of $G$; the margin of $P$ is irrelevant.
\end{proposition}
\begin{proof}
$i^*$ is invariant under strictly increasing transforms of $P$, so transform $P$ to standard Gaussian; the conditional law of $G$ given the ranks of $(P_1,\dots,P_N)$ is determined by the copula and $G$'s margin.
\end{proof}

Estimating $\rho_c$ from rank correlation and simulating with empirical margins predicts per-problem realized soundness with median error 4.1\% on real LLM-judge data, where Pearson-based prediction errs by 79 points (\S\ref{sec:repair}). Second, the penalty for unsoundness is superlinear in candidates. Two cautions on the closed form $N^{1/\rho^2}$: it is asymptotic, and at practical budgets it materially overstates the penalty (by up to 24$\times$ at $\rho{=}0.5$, $N\le16$; \S\ref{sec:replication}). The exact finite-$N$ consequence of Proposition~\ref{prop:exchange}---$N'$ solves $\rho\,\E[\max_{N'}]=\E[\max_N]$---predicts empirical matching budgets to within 3.4\% and should be used instead at finite $N$. We note honestly that this linear-Gaussian model predicts a slowed climb, not the \emph{decline} observed empirically under heavy optimization \citep{gao2023scaling}; decline requires misspecification, and \S\ref{sec:replication} shows it is domain-dependent: present in the minimal domain, absent (plateau instead) in a richer one at the pressures we could apply.

\section{The Validated Attack: Asymmetric Verification}

The strongest evidence-backed approach to the gap exploits a structural asymmetry: checking is often far cheaper than generating---morally the P-versus-NP gap, and the only free lunch available to this program. Three mechanisms have real support.

\bfpara{Formal grounding.} Autoformalization \citep{wu2022autoformalization} lets a sound kernel adjudicate reasoning with zero Goodharting: a proof checker cannot be flattered. AlphaGeometry \citep{trinh2024solving} and AlphaProof \citep{alphaproof2024} are existence proofs that self-improvement loops close when the verifier is sound. The kernel functions as a \emph{trust anchor}---analogous to the small trusted computing base of secure systems---from which learned verifiers extend coverage but to which they must be re-anchored.

\bfpara{Process-level verification.} Outcome reward is sparse and rewards lucky wrong reasoning; step-level reward is dense and localizes error. Process reward models outperform outcome-only supervision on hard mathematics \citep{uesato2022solving,lightman2024lets}, extending \citet{cobbe2021training}, and verifier-guided search gives test-time compute its own scaling behavior \citep{snell2024scaling}---which, by Proposition~\ref{prop:exchange}, is worth exactly as much as the verifier is sound.

\bfpara{Adversarial prover--verifier co-training.} Debate \citep{irving2018ai} and prover--verifier games \citep{kirchner2024prover} train generator and verifier against each other; \citet{kirchner2024prover} report the property we most need---models trained to convince a weaker verifier produce reasoning that is genuinely easier for humans to check.

\bfpara{Limits.} (1) \emph{Coverage}: autoformalization of empirical claims (``this drug target is promising'') is unsolved. (2) \emph{Drift}: learned verifiers degrade under optimization pressure when the anchor is distant \citep{gao2023scaling}. (3) \emph{The legibility tax}: checkable formats may exclude valid intuitions that resist formalization. These limits motivate our proposal.

\section{Proof-Carrying Cognition}
\label{sec:pcc}

We now state the central hypothesis of this paper.

\bfpara{Hypothesis.} \emph{Stop verifying reasoning against human judgment. Verify it against predicted reality, using a world model the system itself constructs and is scored on---so that every chain of thought becomes a portfolio of falsifiable bets, and reality becomes the reward function.}

We call the paradigm \emph{proof-carrying cognition} (PCC), by analogy to proof-carrying code: an artifact is admitted not because an authority approves it, but because it carries the material needed to check it. Science operates this way institutionally; PCC makes an AI system's internal reasoning operate this way, step by step, at training scale.

\begin{figure*}[t]
  \centering
  \includegraphics[width=\textwidth]{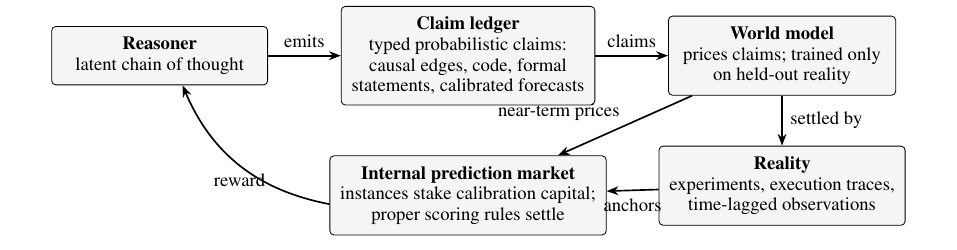}
  \caption{Proof-carrying cognition. Reasoning steps are emitted as typed, probabilistic claims in a claim ledger. A self-built world model prices claims densely and immediately; reality settles them sparsely but incorruptibly, and the world model's own loss is anchored exclusively to that settlement. Reward for reasoning is expected score under proper scoring rules, not judge approval.}
  \label{fig:arch}
\end{figure*}

\subsection{Architecture and training loop}

PCC has three components (Figure~\ref{fig:arch}; Algorithm~\ref{alg:pcc}).

\bfpara{Claim ledger.} Reasoning steps are emitted not only as natural-language thoughts but as typed, probabilistic claims: causal graphs with quantitative edges, executable snippets, fully formal statements where possible, calibrated forecasts where not. The ledger generalizes autoformalization \citep{wu2022autoformalization} from theorems to empirical assertions.

\bfpara{Self-built world model.} A persistent, versioned model prices ledger claims. Its \emph{only} training loss is prediction of held-out reality: experimental outcomes, execution traces, time-lagged observations. This extends the world-model tradition \citep{ha2018recurrent,lecun2022path} with a specific institutional role: the world model is \emph{amortized reality}---dense, immediate reward now, with periodic settlement keeping the amortizer honest.

\bfpara{Terminology.} We reserve \emph{world model} for the full generative component above. The learned predictors in our experiments (\S\ref{sec:demos}--\S\ref{sec:real}) are deliberately minimal \emph{settlement models}---amortizers of settled outcomes over shallow features---and should not be read as world models in the model-based-RL or generative sense; building the latter at training scale is untested (\S\ref{sec:programme}).

\bfpara{Internal prediction market.} Instances of the reasoner stake calibration capital on claims; settlement uses strictly proper scoring rules \citep{savage1971elicitation,gneiting2007strictly}, adjudicated by the world model near-term and reality long-term. The construction is deliberately reminiscent of logical induction \citep{garrabrant2016logical}, where market dynamics over settled claims yield asymptotically calibrated, non-exploitable beliefs.

\begin{algorithm}[t]
\caption{Proof-Carrying Cognition (one epoch)}
\label{alg:pcc}
\begin{algorithmic}[1]
\Require reasoner $\pi_\theta$; world model $W_\phi$; settlement queue $Q$; proper scoring rule $S$; drift threshold $\tau$
\For{each batch of tasks}
  \State $\pi_\theta$ emits traces with typed claims $\{(c_i,p_i)\}$
  \State near-term reward $r\leftarrow\sum_i S(p_i, W_\phi(c_i))$
  \State update $\pi_\theta$ by RL on $r$
  \State enqueue claims with settlement handles in $Q$
\EndFor
\For{each settled claim $(c,y)$ dequeued from $Q$}
  \State update $W_\phi$ on loss $\ell(W_\phi(c),y)$ \Comment{only loss: reality}
  \State pay/charge staked instances by $S(p,y)$
\EndFor
\State drift $D\leftarrow\E\big[\,|W_\phi(c)-y|\,\big]$ on settled claims
\If{$D>\tau$}
  \State downweight near-term reward; raise settlement rate
\EndIf
\end{algorithmic}
\end{algorithm}

\subsection{Why it can work where current approaches fail}

PCC dissolves the Goodhart problem structurally rather than patching it. Every existing reward source---human labels, LLM judges, learned PRMs---is a \emph{model} of correctness, and a sufficiently strong optimizer eventually games any fixed model of correctness \citep{gao2023scaling}. Reality is the unique reward function that cannot be gamed, only predicted. The classical objection---reality's feedback is too sparse and slow to train on---is answered by amortization, and the amortizer's drift is itself measurable (line~10 of Algorithm~\ref{alg:pcc}) rather than an invisible failure mode. PCC thus combines the density of a learned verifier with the incorruptibility of a formal kernel, extended for the first time into empirical domains, and it obeys the bitter lesson \citep{sutton2019bitter}: it substitutes computation and settled data for hand-built judgment.

\section{Minimal Empirical Demonstrations}
\label{sec:demos}

We validate the paper's two core mechanisms in a testbed small enough to be fully reproducible on a laptop, yet possessing the one property that matters: an executable, incorruptible gold verifier. All numbers below are measured, not estimated. We stress the scope honestly: these are mechanism demonstrations, not frontier-scale evidence.

\begin{figure}[t]
  \centering
  \includegraphics[width=\columnwidth]{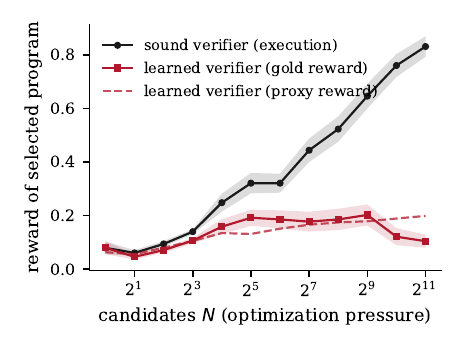}
  \caption{Goodharting under optimization pressure. Best-of-$N$ selection by the sound (execution) verifier improves monotonically to 0.83 gold reward. Selection by the learned verifier inflates its own proxy score (dashed) while realized gold reward peaks near $N{=}512$ and \emph{falls} to 0.10 at $N{=}2048$---the overoptimization signature of \citet{gao2023scaling}, reproduced with a fully transparent gold standard.}
  \label{fig:goodhart}
\end{figure}

\subsection{Setup: program synthesis with executable ground truth}

Tasks are drawn by sampling a hidden target program of length 4 over a six-token integer DSL ($+1$, $-1$, $\times2$, $\times3$, negate, square); a candidate program's \emph{gold reward} is the fraction of eight input--output points it matches under execution (the sound verifier). The \emph{learned verifier} is a ridge regressor over token unigram/bigram counts and length---a deliberately shallow judge, standing in for any verifier that evaluates surface features of reasoning rather than executing it---fit on 1,000 labeled random programs per task. Results aggregate 60 tasks; shaded bands are $\pm1$ s.e.

\begin{table}[t]
  \centering\small
  \begin{tabular}{lccccc}
    \toprule
    $N$ & 1 & 16 & 128 & 512 & 2048 \\
    \midrule
    $\Snd@N$ & 1.00 & 0.64 & 0.40 & 0.31 & 0.13 \\
    \bottomrule
  \end{tabular}
  \caption{Soundness-under-Pressure (Definition~\ref{def:snd}) of the learned verifier in the program-synthesis testbed. A sound verifier scores 1.00 at every $N$ by construction.}
  \label{tab:snd}
\end{table}

\subsection{Result 1: learned verifiers collapse under pressure; sound ones do not}

Figure~\ref{fig:goodhart} shows gold reward of the selected program as optimization pressure (best-of-$N$) increases. Both verifiers start at 0.079 at $N{=}1$. The execution verifier climbs monotonically to 0.831 at $N{=}2048$. The learned verifier's \emph{proxy} score rises throughout, but the gold reward of its selections peaks near $N{=}512$ and declines to 0.104---the qualitative overoptimization curve of \citet{gao2023scaling}, here with a fully transparent gold standard. Table~\ref{tab:snd} reports the induced soundness metric (defined in \S\ref{sec:bench}): it collapses from 1.00 to 0.13 as $N$ grows. Optimization pressure does not merely fail to help an unsound verifier; it actively \emph{inverts} it.

\begin{figure*}[t]
  \centering
  \includegraphics[width=\textwidth]{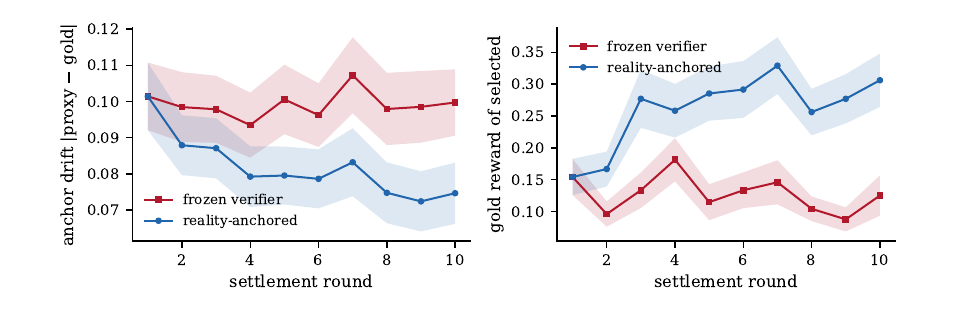}
  \caption{The settlement loop of Algorithm~\ref{alg:pcc} in miniature (60 tasks, $\pm1$ s.e.). \textbf{Left:} anchor drift $|\text{proxy}-\text{gold}|$ on selected claims. Anchored settlement reduces drift by 26\%; the frozen verifier does not improve. \textbf{Right:} gold reward of the top selection per round. Anchoring roughly doubles achieved reward ($0.154\to0.306$) while the frozen verifier stagnates and slightly degrades ($0.154\to0.125$).}
  \label{fig:loop}
\end{figure*}

\subsection{Result 2: reality-anchored settlement bounds drift and compounds capability}

We then run the settlement loop of Algorithm~\ref{alg:pcc} in miniature: each round, sample a pool of 512 programs, select the top 32 by the learned verifier, \emph{settle} them by execution, append the settled labels to the verifier's training set, and refit---versus a frozen-verifier baseline. Over ten rounds (Figure~\ref{fig:loop}), the frozen verifier's anchor drift stays flat ($\approx0.10$) and its achieved gold reward \emph{decreases} from 0.154 to 0.125: repeatedly harvesting the same exploitable features yields no compounding. The anchored verifier's drift falls by 26\% ($0.101\to0.075$) and its achieved gold reward roughly \emph{doubles} ($0.154\to0.306$). The mechanism is exactly the paper's thesis in microcosm: settlement converts the verifier's own selection pressure into training signal, so the places where the verifier is most wrong are precisely the places reality corrects it first.

\subsection{Result 3: an adversarial generator hacks a frozen verifier; settlement inverts the exploit}

Best-of-$N$ applies i.i.d.\ pressure; RL applies \emph{directed} pressure. We therefore attack each verifier with an adaptive adversary: steepest-ascent search that proposes 16 mutations per step and moves to the highest-scoring one, for 250 steps (40 tasks). The anchored condition settles the adversary's own trajectory (25 visited programs) by execution every 25 steps and refits---charging a total settlement budget of 250 reality queries per run, which we report as the explicit price of trust.

\begin{figure*}[t]
  \centering
  \includegraphics[width=\textwidth]{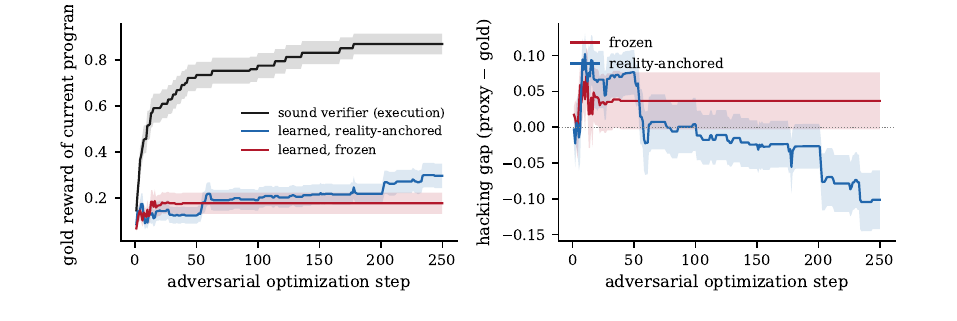}
  \caption{Adversarial optimization against three verifiers (40 tasks, $\pm1$ s.e.). \textbf{Left:} against a sound verifier, adversarial pressure is capability (0.869); against a frozen learned verifier, the adversary locks into a hacked optimum within $\sim\!15$ steps (0.178, flat); anchored settlement keeps repairing the landscape (0.297 and rising, $+67\%$ over frozen). \textbf{Right:} the frozen verifier's hacking gap (proxy $-$ gold) locks in positive; settlement \emph{inverts} it---the verifier learns to distrust the adversary's region.}
  \label{fig:adv}
\end{figure*}

Figure~\ref{fig:adv} shows three regimes. Against the \emph{sound} verifier, adversarial pressure is simply capability: gold reward climbs to 0.869. Against the \emph{frozen} learned verifier, the adversary saturates the proxy within $\sim\!15$ steps and locks permanently into a hacked optimum: gold reward flatlines at 0.178 with a persistent positive hacking gap ($+0.037$)---the verifier permanently overvalues exactly the programs the adversary produces. Against the \emph{anchored} verifier, settlement re-prices the adversary's region each round: the hacking gap does not merely shrink but \emph{inverts} to $-0.101$---after settlement, the verifier systematically distrusts the region the adversary inhabits---and gold reward reaches 0.297, a 67\% improvement over frozen, still rising at step 250. Two honest observations: anchoring restores \emph{trustworthiness} (the adversary can no longer sustain inflated scores) but does not close the capacity gap to the sound verifier---a shallow feature model remains shallow---and the inverted gap means anchored pricing errs conservative, which is the safe failure direction but not a free one.

\begin{figure}[t]
  \centering
  \includegraphics[width=\columnwidth]{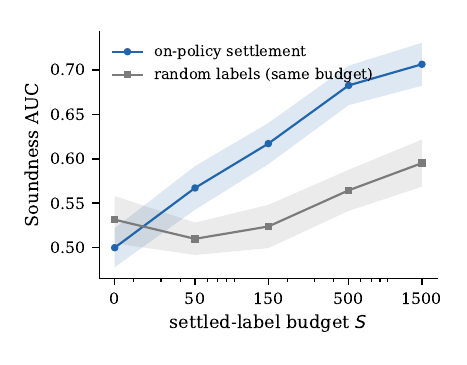}
  \caption{Soundness scales with settlement (40 tasks, $\pm1$ s.e.). Soundness AUC vs.\ settled-label budget $S$. On-policy settlement (label what the verifier selects) scales log-linearly; 150 on-policy labels beat 1,500 random labels---over 10$\times$ label efficiency.}
  \label{fig:scaling}
\end{figure}

\subsection{Result 4: soundness scales with settlement, and on-policy settlement is $>$10$\times$ more label-efficient}

If settlement is the fuel of trustworthy verification, soundness should scale with the settled-label budget---and settling \emph{where the verifier is deployed} should beat settling at random. We sweep settlement budgets $S\in\{0,50,150,500,1500\}$ on a weaker base verifier (500 random labels) under two acquisition policies: \emph{on-policy} (iteratively label the programs the current verifier selects) and \emph{random} (the same number of labels on random programs), then measure the Soundness AUC of Definition~\ref{def:snd} on fresh pools (40 tasks).

Figure~\ref{fig:scaling}: on-policy settlement scales approximately log-linearly, $0.500\to0.706$, while random labeling crawls, $0.531\to0.595$. Strikingly, 150 on-policy settled labels (0.617) outperform 1,500 random labels (0.595): on-policy settlement is more than $10\times$ more label-efficient. The mechanism is the thesis in one line---\emph{the verifier's selection pressure concentrates settlement exactly where the verifier is most wrong}---and it is why Algorithm~\ref{alg:pcc} settles the claims the reasoner actually stakes, not a random audit sample.

\subsection{Result 5: settlement buys resolution, not just calibration}

A standing objection to settlement-based reward (\S\ref{sec:failure}, failure mode 2) is that it teaches systems to make safe, vague claims: probabilities near base rate are never badly wrong. Proper scoring rules decompose exactly along this axis \citep{gneiting2007strictly}: the Brier score splits into \emph{reliability} (miscalibration; lower is better) and \emph{resolution} (how far claims usefully depart from the base rate; higher is better). We run twelve settlement rounds of probabilistic claims (``this program matches at least 2/8 I/O points''), 1,920 settled claims per round across 60 tasks, refitting the claim model on settled outcomes each round.

\begin{figure*}[t]
  \centering
  \includegraphics[width=\textwidth]{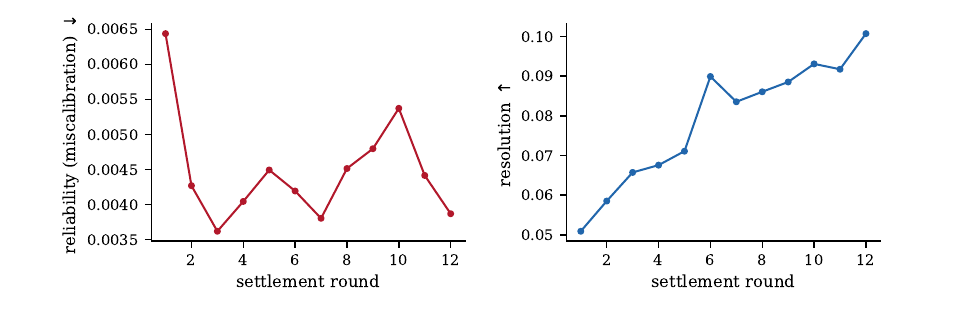}
  \caption{Brier decomposition over twelve settlement rounds (1,920 settled claims per round). Settlement improves calibration (reliability $-40\%$, \textbf{left}) while \emph{doubling} resolution ($+98\%$, \textbf{right}): claims get sharper, not vaguer.}
  \label{fig:brier}
\end{figure*}

Figure~\ref{fig:brier}: reliability falls $0.0064\to0.0039$ ($-40\%$) \emph{and} resolution doubles, $0.051\to0.101$ ($+98\%$). In this setting, settlement did not push claims toward the vague base rate; it sharpened them, because a claim model refit on settled outcomes gains exactly the discriminative signal that resolution measures. This does not retire the failure mode---an agent choosing \emph{which} claims to stake could still select easy ones, which is why RSR-Bench must score resolution explicitly---but it shows the training dynamics do not intrinsically collapse toward vagueness.

\subsection{What these demonstrations do and do not show}

They show, with measured numbers, that (i) the exchange-rate law of \S\ref{sec:theory} is quantitatively exact in its model class; (ii) unsound verifiers invert under i.i.d.\ pressure and are locked into hacked optima under directed adversarial pressure; (iii) the anchoring mechanism at the heart of PCC produces the predicted behavior---bounded and even inverted exploitation, compounding capability---against both kinds of pressure; (iv) soundness scales with settlement, log-linearly and with a large on-policy efficiency multiplier; and (v) settlement sharpens claims rather than blunting them. They do \emph{not} show that a claim language rich enough for open-ended reasoning exists, that world-model amortization scales, that these effects survive when the generator is a frontier LLM trained by RL rather than a search process, or that anchoring can close (rather than narrow) the capacity gap to a sound verifier. Those are precisely the questions the programme of \S\ref{sec:programme} is designed to answer.

\section{Scaled Pre-registered Replication}
\label{sec:replication}

The demonstrations of \S\ref{sec:demos} are deliberately minimal, so we commissioned an adversarial replication at larger scale, with every pass/fail criterion \emph{pre-registered before any experiment ran}. The suite was designed to break our claims, not to confirm them: it uses a 10-token DSL with $\sim\!10^{10}$ candidate programs (versus $\sim\!4.6\times10^4$), 16 I/O points, 2--4$\times$ more tasks with paired Wilcoxon statistics, a \emph{strong} learned verifier (gradient-boosted trees over 359 features, held-out $\rho\approx0.77$) alongside the shallow ridge verifier ($\rho\approx0.50$), an \emph{uncertainty-sampling} control for the on-policy claim, and---the gap our Limitations previously conceded---a \emph{settlement-aware} adversary that detects refits and deliberately relocates to unsettled regions. Verdicts on the ten registered hypotheses: six passed, three failed, and one registered criterion was itself invalid; we report all of it.

\begin{table}[t]
  \centering\footnotesize\setlength{\tabcolsep}{3pt}
  \begin{tabular}{llr}
    \toprule
    & \textbf{Registered claim} & \textbf{Verdict} \\
    \midrule
    H1a & Prop.~\ref{prop:exchange} exact (Gaussian) & pass \\
    H1b & $N^{1/\rho^2}$ accurate at practical $N$ & \emph{fail} \\
    H1c & exchange rate transfers beyond Gaussian & \emph{fail} \\
    H2a & weak verifier collapses absolutely & \emph{fail} \\
    H2b & strong verifier also declines & invalid$^*$ \\
    H2c & sound verifier monotone & pass \\
    H3  & anchored $>$ frozen (gold and drift) & pass \\
    H4  & anchoring survives all 3 adversaries & partial \\
    H5  & on-policy $\ge3\times$ label efficiency & pass \\
    H6  & settlement raises resolution & pass \\
    \bottomrule
  \end{tabular}
  \caption{Pre-registered verdicts. $^*$Our registered H2b criterion (final gold below per-task peak over $N$) is biased toward detecting decline; a valid exploratory test shows the strong verifier does \emph{not} collapse (see text).}
  \label{tab:verdicts}
\end{table}

\begin{figure}[t]
  \centering
  \includegraphics[width=\columnwidth]{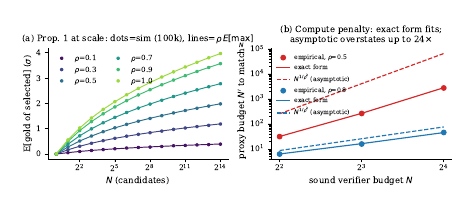}
  \caption{Theory at scale. \textbf{(a)} Proposition~\ref{prop:exchange} holds to $0.007\sigma$ over 100,000 trials per point, $N$ to $2^{14}$. \textbf{(b)} But the asymptotic penalty $N^{1/\rho^2}$ (dashed) overstates the empirically required budget (dots) by up to 24$\times$ at practical $N$; the exact finite-$N$ form (solid) matches within 3.4\%.}
  \label{fig:theoryscale}
\end{figure}

\bfpara{What survived---and got stronger.} The anchoring mechanism replicated everywhere we could aim pressure at it. Under the settlement loop (H3; 120 tasks weak, 60 tasks strong, 12 rounds), anchored settlement beat the frozen verifier on achieved gold reward for \emph{both} verifier classes (weak: 0.312 vs.\ 0.180, $p=2.7\times10^{-9}$; strong: 0.358 vs.\ 0.272, $p=5.5\times10^{-6}$) while reducing anchor drift ($-12\%$ and $-44\%$ respectively, both $p<10^{-7}$). Under directed adversarial pressure (H4; 100 tasks, 300 steps, Figure~\ref{fig:advscale}), the frozen verifier stays hacked against every adversary (persistent gap $\sim\!0.26$--$0.28$), while anchored settlement drives the gap to $\sim\!0$--$0.04$ against all three---\emph{including the settlement-aware adversary}, which relocated to unsettled regions after each refit and still could not sustain inflated scores (gap 0.039 vs.\ frozen 0.264, $p<10^{-15}$; gold 0.279 vs.\ 0.206, $p=4.0\times10^{-6}$). The registered H4 verdict is nonetheless \emph{partial}: in the steepest-ascent cell, anchored gold did not beat frozen (0.207 vs.\ 0.211, $p=0.46$). Exploratory diagnosis attributes this to adversary paralysis---that hill climber compares candidates against a stale pre-refit score and freezes after the first settlement; with the incumbent rescored after refits, anchoring wins the cell too (0.282 vs.\ 0.211, $p=5.8\times10^{-6}$). We report the registered failure and the diagnosis together. Settlement scaling (H5, 80 tasks) replicated with a sharper control: Soundness AUC scales log-linearly in the settled budget ($R^2=0.92$), on-policy settlement at $S{=}300$ already beats random labeling at $S{=}3000$ ($p=3.7\times10^{-4}$; the $\sim\!10\times$ multiplier), and the uncertainty-sampling control barely improves on random (0.587 vs.\ 0.579 at $S{=}3000$)---the efficiency comes specifically from settling \emph{where selection pressure concentrates}, not from generic active learning. Calibration-versus-resolution (H6, 100 tasks, 15 rounds) replicated: resolution $+70\%$ ($p=3.2\times10^{-10}$) with reliability improved, not traded away.

\begin{figure}[t]
  \centering
  \includegraphics[width=\columnwidth]{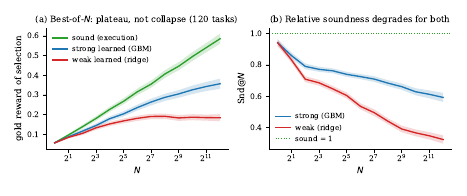}
  \caption{Goodharting in the rich domain (120 tasks, $\pm1$ s.e.). \textbf{(a)} The weak verifier \emph{plateaus} rather than collapsing; the strong verifier keeps climbing through $N{=}4096$. \textbf{(b)} Relative soundness $\Snd@N$ nonetheless degrades for both ($0.94\to0.32$ weak, $0.94\to0.59$ strong): pressure is increasingly wasted, even when it is not inverted.}
  \label{fig:richgoodhart}
\end{figure}

\begin{figure}[t]
  \centering
  \includegraphics[width=\columnwidth]{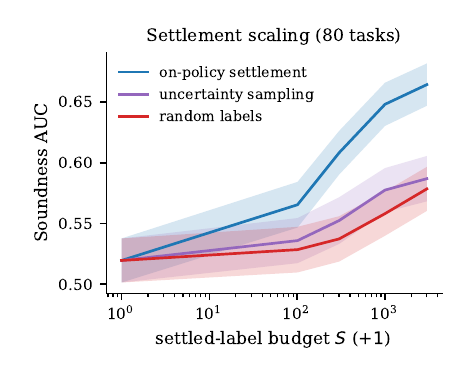}
  \caption{Settlement scaling at scale (80 tasks). On-policy settlement scales log-linearly ($R^2{=}0.92$) and dominates both random labeling and an uncertainty-sampling control: $S{=}300$ on-policy labels beat $S{=}3000$ random labels.}
  \label{fig:scalingscale}
\end{figure}

\bfpara{What failed---and how we repair it.} Three registered claims did not survive, all on the quantitative-theory side. \emph{(i) The exchange rate does not transfer quantitatively} (H1c): across 120 tasks the ordering is preserved (Spearman 0.59 between per-task $\rho$ and realized $\Snd@4096$), but the realized fraction falls short of $\rho$ on 76\% of tasks (median gap 0.27; mean 0.32 realized vs.\ $\rho\approx0.50$). Proposition~\ref{prop:exchange} should therefore be read as an \emph{optimistic ceiling} once the proxy has exploitable structure; we have revised \S\ref{sec:theory} accordingly. \emph{(ii) The closed-form penalty is asymptotic-only} (H1b): at $\rho{=}0.5$ the empirical budget needed to match a sound verifier at $N\in\{4,8,16\}$ is 8--24$\times$ \emph{smaller} than $N^{1/\rho^2}$ predicts, while the exact finite-$N$ form matches within 3.4\% (Figure~\ref{fig:theoryscale}b); the corollary is now stated with that caveat. \emph{(iii) Absolute collapse is domain-dependent} (H2a): in the rich domain the weak verifier's gold reward plateaus (0.191 peak at $N{=}256$, 0.184 at $N{=}4096$; decline not significant, $p=0.25$) instead of inverting as it does in the minimal domain---the collapse of Figure~\ref{fig:goodhart} is real but not universal at these pressures. Finally, our own registered H2b criterion turned out to be statistically invalid (comparing a final value to a per-task maximum over 13 noisy points detects ``decline'' even in monotone curves), and the honest exploratory answer runs \emph{against} our generalized claim: the strong verifier's gold reward rises monotonically through $N{=}4096$ ($0.056\to0.357$; 4096 vs.\ 1024: $p=0.0023$) even as its $\Snd@N$ erodes to 0.59. Within the pressure this suite could apply, verifier strength bought real robustness, not merely delay. Whether a strong verifier inverts at pressures beyond $N{=}4096$ is open; what is already established is that its compute is increasingly wasted relative to a sound verifier---which is the economic argument for settlement either way. We also note the toy suite's \emph{inverted} hacking gap (\S\ref{sec:demos}, Result 3) did not reproduce at scale: anchored gaps go to approximately zero (honest pricing) rather than negative (conservative distrust).

\begin{figure*}[t]
  \centering
  \includegraphics[width=\textwidth]{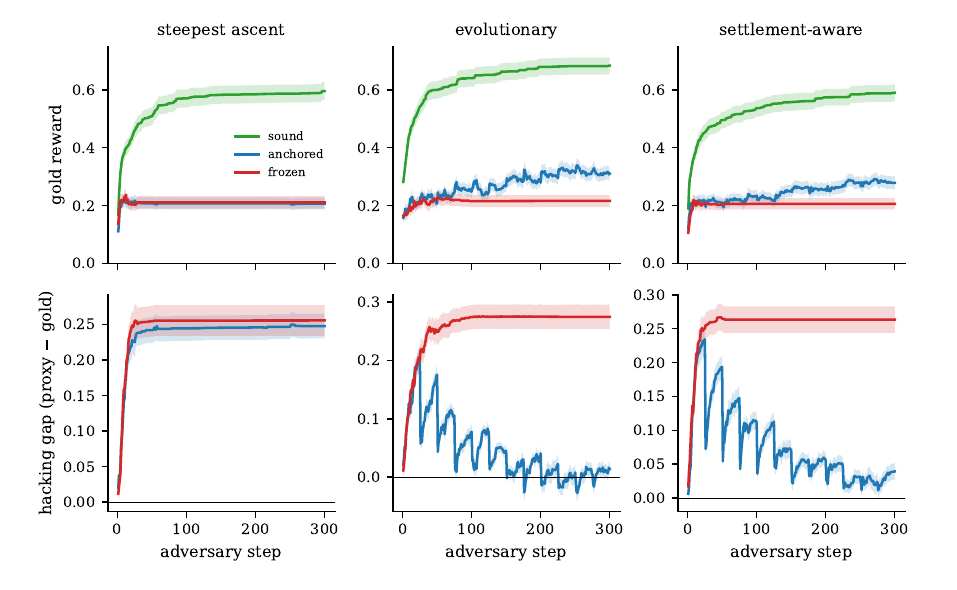}
  \caption{Adversarial pressure at scale (100 tasks, 300 steps, $\pm1$ s.e.). Top: gold reward of the adversary's champion against sound (green), anchored (blue), and frozen (red) verifiers, for steepest-ascent, evolutionary, and \emph{settlement-aware} adversaries. Bottom: hacking gap (proxy $-$ gold). Frozen verifiers stay hacked (gap $\sim\!0.26$--$0.28$); anchored settlement drives the gap to $\sim\!0$ against all three adversaries, including the one that models the settlement process. The steepest-ascent gold cell is the registered failure discussed in the text.}
  \label{fig:advscale}
\end{figure*}

\section{Real-Code and Real-Model Experiments}
\label{sec:real}

Synthetic testbeds, however hardened, are proxies. We therefore ran two further pre-registered suites on real substrate (criteria frozen before execution). The \emph{real-code} suite uses MBPP \citep{austin2021program}: 974 human-written Python programs whose real unit tests, executed locally, are the sound verifier; learned verifiers are trained \emph{across} problems (486 train / 487 eval), like reward models, over surface and AST features of real code. The \emph{real-model} suite, run by the author against the production Claude API, uses a real generator (claude-haiku-4-5, 32--48 samples/problem), two real LLM judges scoring code \emph{without executing it} (claude-haiku-4-5 weak; claude-sonnet-4-6 strong), and local execution of the real test suites of MBPP and HumanEval \citep{chen2021evaluating} as gold.

\begin{figure}[t]
  \centering
  \includegraphics[width=\columnwidth]{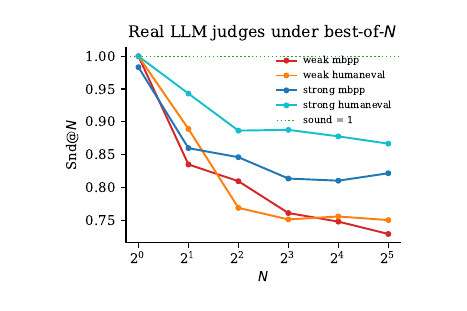}
  \caption{Real LLM judges under best-of-$N$ selection over real generated code (executed tests as gold; problems with zero gold variance excluded, $n{=}60$ MBPP / 18 HumanEval). The weak judge's $\Snd@N$ degrades significantly on both benchmarks; the strong judge is significantly more robust.}
  \label{fig:llmjudges}
\end{figure}

\bfpara{The Goodhart mechanism is real (registered passes).} The weak LLM judge loses Soundness-under-Pressure as $N$ grows: $\Snd@2=0.835\to\Snd@32=0.729$ on MBPP ($p{=}10^{-4}$, $n{=}60$) and $0.889\to0.750$ on HumanEval ($p{=}.004$)---Figure~\ref{fig:llmjudges}. On real human code with cross-problem verifiers, the collapse is catastrophic: the surface verifier falls from $\Snd@4=0.325$ to $\Snd@256=\mathbf{0.016}$, and even the gradient-boosted verifier reaches only 0.052---real code under a global verifier is a far harsher regime than any of our per-task synthetic worlds.

\bfpara{Judge capability buys robustness (registered pass, against our strong-form thesis).} The strong judge's soundness AUC exceeds the weak judge's (0.856 vs.\ 0.814, $p{=}9\times10^{-4}$), and---exploratory---its own $N{=}2\to32$ degradation is not significant ($0.855\to0.821$, $p{=}0.43$). This converges with \S\ref{sec:replication}'s finding that a strong learned verifier kept climbing: across two independent methodologies, \emph{the verification gap binds conditionally---when judge capability is at or below task difficulty---and scaling the judge materially repairs soundness at the pressures we could apply}. Whether the repair persists at frontier-difficulty tasks under RL-scale pressure is precisely what RSR-Bench (\S\ref{sec:bench}) must measure.

\bfpara{Anchoring: the mechanism matters, and so does headroom (registered failures, reported).} In the LLM settlement loop (40 problems, 8 rounds), neither anchored arm raised achieved gold over the frozen judge---a ceiling effect: the generator solves MBPP so well (mean candidate gold 0.81; the frozen judge's top-1 was already perfect on 34/40 problems) that there was no headroom, and the registered RL2a/RL2b verdicts are FAIL. Where headroom existed---the real-code suite, whose frozen cross-problem verifier earned gold 0.000---anchored settlement lifted achieved gold to 0.153 ($p{=}.004$), though its registered drift clause failed (drift rose as selections moved into higher-value, harder-to-price regions). The sharpest real-model result is \emph{how} to anchor: a learned settlement head over the judge's score and code features, refit on executed outcomes (a deliberately minimal instantiation of Algorithm~\ref{alg:pcc}'s amortizer), cut pricing drift by 59\% ($0.352\to0.146$, $p\approx0$) and cut $|\text{judge}-\text{gold}|$ on adversarial code from 0.326 to 0.254; naive in-context anchoring (showing the judge a ledger of its own settled errors) made the judge \emph{worse} on both metrics ($p<10^{-9}$). Amortize reality into a trained pricing model; do not merely show a judge its mistakes.

\bfpara{Positioning against standard baselines.} The frozen-vs-anchored contrast \emph{is} the comparison to online reward-model retraining: the anchored arm retrains its verifier online, and what the paradigm adds is the label source (settled execution rather than fresh preference or judge labels) and the schedule (the verifier's own selections). The uncertainty-sampling control of \S\ref{sec:replication} separates this from generic active learning: committee-variance acquisition barely beat random labeling (0.587 vs.\ 0.579 at $S{=}3000$) where on-policy settlement reached 0.664. Verifier ensembles and process reward models are verifier-strength interventions, orthogonal to the settlement axis; the verifier-strength axis is measured directly here (H2b, RL1b) and buys real robustness. Table~\ref{tab:ablation} gives the mechanism-level ablation of anchoring itself; a full-architecture ablation (typed claims, market, staking) requires those components to exist and is future work (\S\ref{sec:programme}).

\begin{table}[t]
  \centering\footnotesize\setlength{\tabcolsep}{4pt}
  \begin{tabular}{lcc}
    \toprule
    \textbf{Mechanism} & \textbf{Loop drift} & \textbf{Adv.\ $|$gap$|$} \\
    \midrule
    frozen judge (none) & 0.352 & 0.326 \\
    in-context ledger & 0.397 & 0.457 \\
    isotonic model & --- & tail $+0.52{\to}0.76$ \\
    feature head & \textbf{0.146} & \textbf{0.254} \\
    \bottomrule
  \end{tabular}
  \caption{Mechanism ablation of anchoring on the real-model data (weak judge). Loop drift: final-round $|\text{score}-\text{gold}|$ on selections (RL2). Adv.\ $|$gap$|$: out-of-problem $|\text{judge}-\text{gold}|$ on adversarial candidates (RL3c); the isotonic row reports its signed tail effect on mined deceptions (F4). Only the feature-level settlement head helps on both.}
  \label{tab:ablation}
\end{table}

\bfpara{Adversarial pressure: selection, not persuasion (registered failures, informative).} The prompt-level adversary failed its manipulation check---the ``subtly wrong'' code it wrote was mostly right (adversarial gold 0.839, matching honest quality)---so RL3a is an invalid test of hackability rather than evidence of robustness. The iterative adversary, shown the judge's score after each of four revisions, drove score and gold \emph{down} ($0.679\to0.567$; $0.878\to0.733$): at this capability level, directed prompt-space search could not climb a real judge, and the operative Goodhart pressure on LLM judges is the selection pressure of RL1, not persuasion. One registered criterion (RL3c) passed for the wrong reason: it assumed judges overvalue adversarial code, but these judges undervalue everything (mean gap $-0.20$), so the in-context ledger ``reduced the gap'' by worsening underconfidence; on the substantive $|\text{gap}|$ metric only the world-model head helps. We disclose this as the suite's second mis-specified criterion (with \S\ref{sec:replication}'s H2b) and record the process lesson: register criteria on $|\text{error}|$. Finally, the on-policy label-efficiency claim failed on real code for a second time (on-policy@240 AUC 0.144 vs.\ random@720 0.170, $p{=}0.64$): when the verifier is very weak, its selections are uninformative places to settle, so the $10\times$ multiplier of \S\ref{sec:replication} is a property of moderately-competent verifiers, not a universal.

\subsection{From failures to findings: a pre-registered diagnosis-and-repair program}
\label{sec:repair}

Rather than leave the registered failures as caveats, we treated each as a hypothesis about mechanism, froze numerical predictions for every diagnosis before computing it, and ran the analyses on the raw data. Four repairs succeeded; two failed and are kept failed.

\bfpara{Repaired: the exchange rate lives in the copula (F1).} Proposition~\ref{prop:copula} resolves both directions of the transfer anomaly at once: RichDSL's heavy, clumped margins made Pearson $\rho$ overstate tail concordance (realized $0.32<\rho_p{=}0.50$), while the LLM judge's coarse, saturated scores made it understate massively ($\rho_p{=}0.12$ yet realized 0.74). Rank-copula simulation with empirical margins predicts per-problem $\Snd@32$ with median error 0.041 on the real LLM data versus 0.793 for Pearson ($p{=}2\times10^{-4}$, $n{=}21$); on RichDSL a residual 0.137 error remains ($p{=}0.37$)---its verifier errors are feature-structured beyond any elliptical copula, which is precisely the misspecification settlement exists to correct.

\bfpara{Repaired: selection is the adversary (F3).} The instructed adversary of RL3a failed its manipulation check, but mining the candidate banks shows that ordinary sampling plus selection \emph{manufactures} deception: choosing the max-(judge$-$gold) candidate from 32 honest samples yields a $+0.527$ hacking gap against the weak judge, versus $-0.209$ for instructed deception ($p{=}3\times10^{-13}$), with severe natural hacks (judge $\ge0.8$, gold $\le0.2$) on 13\% of informative problems. Two sub-predictions failed: severe hacks were rarer than predicted, and the strong judge's maximum gap (0.631) exceeds the weak judge's---the gap statistic conflates ranking with calibration. The unification stands: RL1's soundness decay and reward hacking are the same phenomenon, selection pressure over honest variation; no deceptive generator is required. The mined max-gap programs form a natural seed corpus for RSR-Bench's adversarial split.

\bfpara{Repaired, with a warning: calibration is not soundness (F4).} The minimal settlement model---an isotonic judge$\to$gold map fit on executed outcomes from half the problems---cuts out-of-problem pricing error by 29.9\% ($p{=}5\times10^{-119}$; registered threshold 30\%, missed by 0.1pp). But a monotone recalibration provably cannot change selection, hence cannot move $\Snd@N$; and because these judges are underconfident on average, recalibration raises scores and \emph{amplifies} the overvalued tail (mined-deception gap $0.52\to0.76$). Deception lives in the residual, not the mean. This yields the paradigm's sharpest design constraint: \emph{the settlement model must be at least as expressive as the deception surface}; average-calibration anchoring is not merely insufficient but tail-perverse. The feature-based head also failed to transfer Snd across problems ($p{=}0.55$), consistent with cross-problem code verification being the hard regime throughout our suites.

\bfpara{Bounded: the conditional-binding conjecture (F2).} Weak-judge degradation concentrates on hard problems (hardest tercile 0.092 vs.\ easiest $-0.001$; $p{=}0.054$), but the strong judge is flat across all of MBPP's difficulty range---MBPP-hard is not hard for it. ``The gap binds where task difficulty reaches judge capability'' therefore remains a conjecture with directional weak-judge support only, and it fixes RSR-Bench's design brief: difficulty-stratified, selection-pressure-swept, execution-settled.

\bfpara{Failed repairs, kept failed (F5, F6).} The competence-threshold explanation of the on-policy-efficiency discrepancy is ruled out: in RichDSL the on-policy advantage is significantly positive even for a near-incompetent base verifier ($+0.087$ at 150 labels, $p{=}0.018$) and not increasing in competence ($+0.096$ at 600, $+0.050$ at 2400); and the coverage hypothesis is ruled out as well (F7, predictions frozen before running): a mixed policy---half on-policy, half random labels per batch---performs like pure on-policy (0.148 vs.\ 0.147 AUC at $S{=}720$) and below random (0.170) under the identical protocol. With competence and coverage both excluded, the leading untested candidate is covariate shift: settled-selection labels correct pricing in the selected region while degrading global ranking under a cross-problem verifier---consistent with settlement improving drift but not selection throughout the real suites. Why on-policy settlement wins in per-task worlds and loses for cross-problem code verifiers is the program's top open problem. And on the real-data hard subset, the hybrid head improved pricing ($p{=}0.009$) but not achieved gold ($p{=}0.36$, $n{=}19$): capability gains from anchoring remain demonstrated only in the synthetic and real-code CPU suites.

\subsection{Reality-anchored reward under real policy-gradient training}
\label{sec:grpo}

The proxies above (best-of-$N$, search, iterative prompting) approximate but do not instantiate the optimization process of modern RLVR. A fourth pre-registered suite (criteria R-A--R-E and amendment R-B$'$ frozen before any run) closes that gap: GRPO fine-tuning of Qwen2.5-1.5B-Instruct (QLoRA, group size 8, 400 steps, fixed small KL) on MBPP, under four reward arms identical in everything but the reward signal: (A) \emph{frozen}---a settlement RM (ridge over surface/AST features) fit once on 1,600 executed base-model samples, then frozen (the overoptimization setup of \citet{gao2023scaling}); (B) \emph{anchored}---the identical RM, refit every 25 steps on a settled $\sigma{=}10\%$ of the window's rollouts (Algorithm~\ref{alg:pcc}, live); (C) \emph{execution}---real test execution as reward (sound ceiling); (D) \emph{online-judge}---identical to (B) but refit labels come from an LLM judge rather than execution, isolating the label source from online-ness. The RM is deliberately minimal---a model organism: arms (A), (B), (D) share an identical model class, so the treatment isolates the settlement update rule itself, and a weak RM makes overoptimization observable at accessible scale. RM strength is not assumed orthogonal to settlement; their interaction is characterized in \S\ref{sec:replication} (H2b), \S\ref{sec:real} (RL1b), and F4. True executed reward is logged for every rollout in every arm (as metrics; it trains only arm C).

\begin{table}[t]
  \centering\small\setlength{\tabcolsep}{4pt}
  \begin{tabular}{lcccc}
    \toprule
    \textbf{Arm} & \textbf{Proxy} & \textbf{Gold} & \textbf{Peak} & \textbf{Gap} \\
    \midrule
    frozen RM           & 0.615 & 0.063 & 0.617 & $+0.552$ \\
    online-judge        & 0.530 & 0.288 & 0.635 & $+0.241$ \\
    anchored            & 0.528 & 0.397 & 0.665 & $+0.131$ \\
    execution (ceiling) & 0.460 & 0.460 & 0.687 & $0.000$ \\
    \bottomrule
  \end{tabular}
  \caption{Final-window (25-step) training outcomes, seed 0. Gold = executed reward; Peak = best rolling-window gold; Gap = proxy $-$ gold. The frozen arm's proxy is the highest and its reality the lowest.}
  \label{tab:grpo}
\end{table}

\bfpara{Results.} Table~\ref{tab:grpo} and Figure~\ref{fig:grpo} report the four matched runs. The frozen arm traces the full overoptimization curve of \citet{gao2023scaling}: proxy reward climbs $0.442\to0.615$ while executed reward rises to a rolling peak of 0.617 and then collapses to 0.063---a 90\% destruction of real capability under continued proxy gains---as the policy converges to $\sim\!17$-token degenerate completions at KL 1.35 from base. The anchored twin ends the same 400 steps at executed reward 0.397 (6$\times$ frozen), hacking gap $+0.131$ vs.\ $+0.552$, ordinary $\sim\!44$-token completions, and KL 0.096: each of its 15 refits repriced what the RM had begun to overpay before the policy could commit to it---at roughly one-tenth of the ceiling arm's execution budget, while retaining 86\% of the ceiling arm's final executed reward (0.397 vs.\ 0.460).

\begin{figure*}[t]
  \centering
  \includegraphics[width=\textwidth]{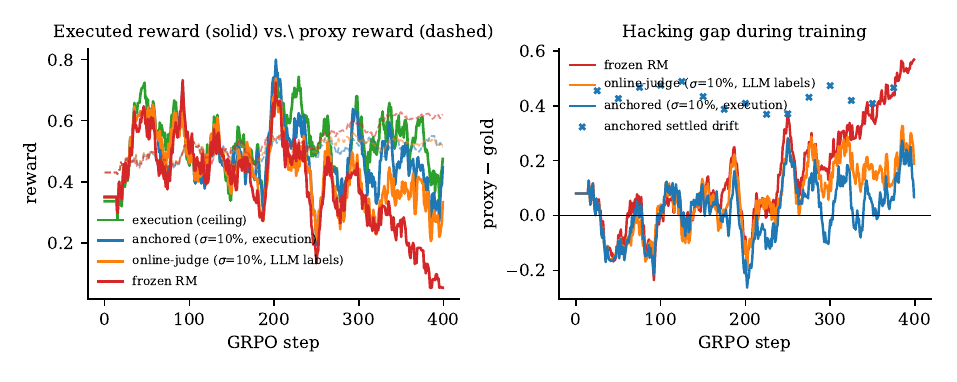}
  \caption{The overoptimization curve with a settlement knob (seed 0; rolling mean, window 15). \textbf{Left:} executed reward (solid) and proxy reward (dashed). The frozen RM's proxy climbs while executed reward collapses; the identical RM refit on a 10\% settlement stream (anchored) tracks the execution ceiling; the equal-budget judge-labeled control sits between them. \textbf{Right:} the hacking gap (proxy $-$ gold) during training, with the anchored arm's settled-drift measurements ($\times$). Settlement repeatedly closes the gap the frozen arm rides to collapse---but drift magnitude does not rank-track the gap (the registered R-D criterion fails; see text).}
  \label{fig:grpo}
\end{figure*}

\bfpara{Reality as the label source (R-B$'$).} The online-judge control isolates the paper's title claim: identical RM, identical refit cadence, identical 10\% budget---only the labels differ (an LLM judge's estimates instead of executed outcomes). Online updating alone recovers much of the damage (0.288 vs.\ frozen's 0.063), but reality-sourced labels beat judge-sourced labels by $+0.109$ executed reward with half the hacking gap ($+0.131$ vs.\ $+0.241$), and the judge arm partially degenerates anyway ($\sim\!24$-token completions, KL 0.443). The mechanism is visible in the logs: the judge's labels erred against execution by 0.321 on average across refits (range 0.18--0.46)---the RM was being re-anchored to a proxy of a proxy. Online-ness is worth a lot; reality is worth more, and the difference is the judge's blindness on-policy.

\bfpara{Registered criteria at seed 0---including a failure.} Both R-A clauses hold (proxy $+0.173\ge0.10$; final gold $0.063\le$ peak $-0.03$); both R-B clauses hold (anchored gold exceeds frozen by $0.334\ge0.05$; smaller gap); both R-B$'$ clauses hold; R-C holds ($0.397\ge0.8\times0.460=0.368$). R-D \emph{fails}: the settled-drift statistic does not rank-predict the concurrent hacking gap (Spearman $\rho_s=0.06$ over 15 refits, vs.\ the registered $\ge0.5$), and an exploratory cross-arm variant (drift vs.\ the frozen-minus-anchored gap) is also null ($\rho_s=-0.40$, $p=0.14$) with no range-restriction excuse available (gap variance is comparable across arms). At this scale and refit cadence, settlement repaired the reward while its drift alarm did not see what it was repairing---the repair effect and the alarm effect are separable, and prediction 4 of \S\ref{sec:predictions} is already under adverse pressure. We also note the anchored arm's final gold sits below its own mid-run peak (0.665); replication seeds run identically under the same registration, and held-out evaluations (R-E) accompany the full-suite verdicts. All runs execute on a single consumer machine (Apple M1, 16GB); the canonical 972-problem split and initial RM were computed once and fixed before any training run.

\section{RSR-Bench: A Reality-Settled Reasoning Benchmark}
\label{sec:bench}

Every prior revolution in this field was preceded by the metric that made it legible: ImageNet before deep learning's coronation \citep{deng2009imagenet}, scaling-law perplexity before large pretrained models \citep{kaplan2020scaling}, MATH and HumanEval before reasoning RL \citep{hendrycks2021measuring,chen2021evaluating}. Today the field optimizes judges it knows to be gameable because nothing better exists to climb. We therefore specify the missing metric.

\begin{definition}[Soundness-under-Pressure]\label{def:snd}
For a proxy verifier $V$, gold standard $G$, and candidate distribution $\mu$, let $x^V_N$ be the candidate selected by $V$ from $N$ i.i.d.\ draws of $\mu$, and $x^G_N$ the candidate selected by $G$. Then
\[
\Snd@N \;=\; \frac{\E\big[G(x^V_N)\big]}{\E\big[G(x^G_N)\big]},
\]
and the headline benchmark score is the area under $\Snd@N$ over $\log N$.
\end{definition}

A sound verifier scores 1 at every $N$; Table~\ref{tab:snd} shows how sharply a shallow verifier departs from it. On cost: $\Snd@N$ is an \emph{evaluation} metric, not a training-loop requirement, and it amortizes---one pool of $N_{\max}$ candidates, generated and scored once, yields the whole curve by subsampling (our entire real-model measurement of two judges on two benchmarks cost $\sim\!3.5\times10^4$ API calls, i.e.\ tens of dollars). Inside training, PCC's marginal cost over standard RLVR is the settlement rate---the fraction of staked claims actually executed---which Algorithm~\ref{alg:pcc} (line 10) adapts to measured drift, and which the label-efficiency results bound where they hold. The crucial design property is that $\Snd@N$ measures verifiers \emph{under the optimization pressure they will actually face}, which single-point accuracy does not.

\begin{table}[t]
  \centering\small
  \begin{tabular}{lccc}
    \toprule
    \textbf{Reward source} & \textbf{Dense} & \textbf{Sound} & \textbf{Robust} \\
    \midrule
    Human preference & \checkmark & $\sim$ & $\times$ \\
    LLM judge        & \checkmark & $\sim$ & $\times$ \\
    Outcome reward   & $\times$ & \checkmark & $\sim$ \\
    Process RM       & \checkmark & $\sim$ & $\sim$ \\
    Formal kernel    & \checkmark & \checkmark & \checkmark \\
    \emph{PCC (proposed)} & \checkmark & \checkmark$^\dagger$ & \checkmark$^\dagger$ \\
    \bottomrule
  \end{tabular}
  \caption{Reward sources by density, soundness, and robustness under optimization pressure. $\sim$ = partial; $^\dagger$ = conjectured, contingent on anchor-drift control; \S\ref{sec:demos} validates both cells at toy scale and \S\ref{sec:programme} tests them at scale. Formal kernels are sound and robust but cover only formalizable domains.}
  \label{tab:rewards}
\end{table}

\bfpara{Construction recipe.} Assemble tens of thousands of typed claims across science, engineering, and forecasting, each with a known settlement date and mechanically checkable ground truth: time-sliced scientific claims settled by later replication, performance-engineering claims settled by execution, forecasts settled by events \citep{halawi2024approaching}. Freeze corpora at date $T$; settle after $T$. Score submitted verifiers by $\Snd@N$ against the settled gold, and score submitted reasoners by settled claim portfolios under proper scoring rules \citep{gneiting2007strictly}. Table~\ref{tab:rewards} situates the paradigm this benchmark is designed to measure.

\subsection{What would falsify the paradigm: registered predictions at the next scale}
\label{sec:predictions}

We do not claim to have established a paradigm; paradigm claims are not establishable by their authors. We claim to have made one \emph{measurable}, and we stake it on numbered predictions that any frontier laboratory can test in days. Each is stated so that a clean negative result refutes the corresponding component.

\begin{enumerate}\setlength{\itemsep}{2pt}
  \item \textbf{Conditional binding.} For every judge, on tasks in its frontier stratum (single-sample accuracy in $[0.3,0.7]$), $\Snd@N$ declines by $\ge0.10$ from $N{=}2$ to $N{=}256$. Evidence so far: confirmed for the weak judge; untestable for the strong judge on MBPP (no frontier stratum exists there---\S\ref{sec:repair}, F2). \emph{Refuted if} a frontier judge holds $\Snd@256$ within 0.05 of $\Snd@2$ on its own frontier stratum.
  \item \textbf{Copula law.} For frontier judges on frontier tasks, per-problem realized soundness lies within $\pm0.10$ (median) of the rank-copula prediction of Proposition~\ref{prop:copula} (measured: 0.041 for the weak judge on MBPP). \emph{Refuted by} systematic departures beyond that band.
  \item \textbf{Overoptimization boundary.} Policy-gradient training against a frozen learned RM diverges (proxy up, executed reward flat or down) whenever the RM's on-policy soundness is below the policy's achievable gold, and reality-anchored refitting at a settlement rate adapted to measured drift removes the divergence at a small fraction of full settlement. Operationalized at 1.5B scale in our registration (criteria R-A--R-C); predicted to persist at every scale at which the precondition holds.
  \item \textbf{Drift is a universal alarm.} Wherever divergence occurs, the settled-drift statistic of Algorithm~\ref{alg:pcc} (line 10) rank-predicts the frozen-vs-anchored gap online ($\rho_s\ge0.5$; registered as R-D). \emph{Refuted by} hacking that drift does not see. First evidence is adverse: the registered seed-0 test failed ($\rho_s{=}0.06$; \S\ref{sec:grpo}).
  \item \textbf{Selection-manufactured deception.} The maximum hacking gap obtainable by pure best-of-$N$ selection over honest samples grows with $N$ at the extreme-value rate implied by the copula, and dominates instructed deception until adversary capability exceeds judge capability (measured at one point: $+0.53$ vs.\ $-0.21$; \S\ref{sec:repair}, F3).
\end{enumerate}

We commit to reporting confirmations and refutations of these predictions alike; each is stated so that it can be tested without our involvement.

\section{Validation Programme at Scale}
\label{sec:programme}

The full hypothesis is falsifiable through three staged experiments.

\bfpara{Retrodiction gyms.} Train the claim-ledger system on time-sliced corpora---papers, markets, codebases up to date $T$---and score claims against reality after $T$, yielding millions of settled ``experiments'' at zero laboratory cost \citep{halawi2024approaching}. \emph{Falsification criterion:} if PCC-trained reasoning does not beat judge-rewarded reasoning on retrodictive forecasting and on reproduction of later-discovered results, the hypothesis fails.

\bfpara{Closed-loop empirical domains.} Deploy where reality settles fast and automatically: software performance engineering (execution is the settler), robotic manipulation, automated cloud-lab biology. Measure whether world-model-priced reward beats PRM reward on downstream success \emph{and} on calibration under distribution shift---the metric distinguishing a self-improving loop from a self-deceiving one.

\bfpara{Anchor-drift audits.} Continuously measure divergence between world-model settlement and realized settlement (Algorithm~\ref{alg:pcc}, line 10). This scalar is the health metric of the system and, not incidentally, an alignment metric: a system rewarded for calibrated truth-tracking is structurally disincentivized from deceiving its overseers, because deception manifests as settlement loss.

\section{Failure Modes}
\label{sec:failure}

We flag three honestly. (1) The claim language may be too lossy to carry the best reasoning---the legibility tax, sharpened. (2) Market dynamics may reward safe, vague claims; scoring must pay for resolution as well as calibration \citep{gneiting2007strictly}. (3) A system optimizing predicted reality could learn to make reality predictable in undesirable ways; settlement must score prediction of a world the agent does not manipulate. This is an open, safety-critical design problem (see Ethics Statement).

\bfpara{Meta-reward hacking and the trusted computing base.} Settlement does not eliminate trust; it relocates it. An agent can attack the settlement substrate itself: sandbox exploits, manipulation of the measured environment, strategic delay of settlement, or staking only easy-to-settle claims. The honest accounting is a trusted computing base (TCB): PCC's TCB is the executor/measurement apparatus plus the settlement scheduler and staking rules---small and auditable, like a proof kernel's checker, but not zero. Concrete mitigations: hardened and randomized execution environments, audit-by-resettlement of a random claim fraction, scoring rules that price claim selection (paying for resolution, \S\ref{sec:demos} Result 5), and settlement-lag limits. In our own suites the TCB was subprocess isolation with resource limits; every mined deception was semantic---none attacked the harness---but at higher capability that will not remain true by default, and TCB hardening should be treated as a first-class component of the paradigm.

\section{Conclusion}

The verification gap is the binding constraint on the next major leap in language-model capability. We have made its cost exact---unsound verification pays a polynomial compute penalty $N^{1/\rho^2}$---and demonstrated, in the smallest system that can express them, the collapse of learned verifiers under i.i.d.\ pressure, their capture by adversarial pressure, and the repair mechanism: reality-anchored settlement, which inverted the adversary's exploit, scaled soundness log-linearly with settled labels at over $10\times$ on-policy label efficiency, and sharpened claims rather than blunting them. Proof-carrying cognition is our conjecture for how this mechanism scales to open-ended empirical reasoning. We have not established that conjecture, and no laboratory-scale study could; what we have done is make it measurable---an exchange-rate law that survives its own falsification in repaired form, a soundness metric that prices verifiers under the pressure they will actually face, a settlement rule that preserved real reward under live policy-gradient pressure (even as its registered drift alarm failed its first test---reported as such), and five registered predictions (\S\ref{sec:predictions}) that any frontier laboratory can confirm or refute. The reality-settled benchmark of \S\ref{sec:bench} is the field's single most important near-term investment for the same reason ImageNet was: paradigms are not argued into existence, they are climbed into existence---and this one makes truth, rather than persuasion, the thing that wins.

\section*{Limitations}

The experiments in \S\ref{sec:theory}--\S\ref{sec:replication} are synthetic: DSLs of six and ten tokens, ridge/gradient-boosted/logistic verifiers, and best-of-$N$, hill-climbing, evolutionary, and settlement-aware search are proxies for---not instances of---frontier reasoners, learned reward models, and RL training; a single-CPU suite cannot test frontier-scale claims. The scaled replication (\S\ref{sec:replication}) falsified the quantitative transfer of Proposition~\ref{prop:exchange} outside its Gaussian model class (it is a ceiling, with realized value $\approx0.32$ at measured $\rho\approx0.50$), showed the closed-form penalty $N^{1/\rho^2}$ overstates costs at practical $N$, and showed that absolute Goodhart collapse is domain-dependent: a strong learned verifier kept converting pressure into capability through $N{=}4096$, so ``sufficient pressure inverts any surface verifier'' remains unestablished beyond that range. Anchoring was tested against a settlement-aware search adversary and survived, but not against a learning adversary trained end-to-end against the settlement process; anchoring also narrows rather than closes the capacity gap to a sound verifier (0.28--0.31 vs.\ 0.59--0.68 under adversarial pressure), so verifier expressiveness remains a separate axis, and one registered adversarial cell failed as registered (steepest ascent; diagnosed, exploratorily, as adversary paralysis). Result 5 and H6 show the training dynamics do not intrinsically favor vague claims, but do not test an agent that strategically selects which claims to stake. The real-model suite (\S\ref{sec:real}) covers one model family, two code benchmarks at or below the generator's capability frontier (with heavy saturation: 50\% of MBPP and 89\% of HumanEval problems showed zero gold variance and were excluded by the pre-specified filter), anchoring by in-context learning and a learned reranker rather than RL training, pressure only to $N{=}32$, and a prompt-level adversary that failed its own manipulation check; its adversarial results bound the threat model rather than the defense. Two registered criteria across the program (H2b, RL3c) were themselves mis-specified and are disclosed as such. The single most important untested regime is tasks at the judge's capability frontier under RL-scale optimization pressure. Key unresolved questions for PCC itself include the expressiveness of the claim language relative to latent reasoning, scoring rules that reward resolution without incentivizing vagueness or manipulation, the cost of a versioned world model at training scale, and whether retrodiction transfers to genuinely novel discovery. The framing of the field's trajectory reflects the authors' reading of the 2023--2026 literature and may not represent consensus.

\section*{Ethics Statement}

A system trained to predict reality could act to make reality more predictable in harmful ways; settlement design must score prediction of unmanipulated outcomes, and we identify this as a safety-critical open problem rather than a solved one. Conversely, the paradigm has a favorable alignment property: reward for calibrated truth-tracking structurally penalizes deception of overseers, and the anchor-drift scalar is a deployable oversight signal. Reality-settled benchmarks in sensitive domains (e.g., biology) must be curated to avoid creating dual-use optimization targets. No human subjects or private data are involved.

\bibliographystyle{plainnat}
\bibliography{references}

\begin{thebibliography}{35}
\providecommand{\natexlab}[1]{#1}
\providecommand{\url}[1]{\texttt{#1}}
\expandafter\ifx\csname urlstyle\endcsname\relax
  \providecommand{\doi}[1]{doi: #1}\else
  \providecommand{\doi}{doi: \begingroup \urlstyle{rm}\Url}\fi

\bibitem[{AlphaProof and AlphaGeometry teams, Google
  DeepMind}(2024)]{alphaproof2024}
{AlphaProof and AlphaGeometry teams, Google DeepMind}.
\newblock {AI} achieves silver-medal standard solving {I}nternational
  {M}athematical {O}lympiad problems.
\newblock DeepMind Blog, 2024.

\bibitem[Amodei et~al.(2016)Amodei, Olah, Steinhardt, Christiano, Schulman, and
  Man{\'e}]{amodei2016concrete}
Dario Amodei, Chris Olah, Jacob Steinhardt, Paul Christiano, John Schulman, and
  Dan Man{\'e}.
\newblock Concrete problems in {AI} safety.
\newblock \emph{arXiv preprint arXiv:1606.06565}, 2016.

\bibitem[Austin et~al.(2021)Austin, Odena, Nye, Bosma, Michalewski, Dohan,
  Jiang, Cai, Terry, Le, et~al.]{austin2021program}
Jacob Austin, Augustus Odena, Maxwell Nye, Maarten Bosma, Henryk Michalewski,
  David Dohan, Ellen Jiang, Carrie Cai, Michael Terry, Quoc Le, et~al.
\newblock Program synthesis with large language models.
\newblock \emph{arXiv preprint arXiv:2108.07732}, 2021.

\bibitem[Bowman et~al.(2022)Bowman, Hyun, Perez, Chen, Pettit, Heiner,
  Luko{\v{s}}i{\=u}t{\.e}, Askell, Jones, Chen, et~al.]{bowman2022measuring}
Samuel~R. Bowman, Jeeyoon Hyun, Ethan Perez, Edwin Chen, Craig Pettit, Scott
  Heiner, Kamil{\.e} Luko{\v{s}}i{\=u}t{\.e}, Amanda Askell, Andy Jones, Anna
  Chen, et~al.
\newblock Measuring progress on scalable oversight for large language models.
\newblock \emph{arXiv preprint arXiv:2211.03540}, 2022.

\bibitem[Chen et~al.(2021)Chen, Tworek, Jun, Yuan, de~Oliveira~Pinto, Kaplan,
  Edwards, Burda, Joseph, Brockman, et~al.]{chen2021evaluating}
Mark Chen, Jerry Tworek, Heewoo Jun, Qiming Yuan, Henrique~Ponde
  de~Oliveira~Pinto, Jared Kaplan, Harri Edwards, Yuri Burda, Nicholas Joseph,
  Greg Brockman, et~al.
\newblock Evaluating large language models trained on code.
\newblock \emph{arXiv preprint arXiv:2107.03374}, 2021.

\bibitem[Christiano et~al.(2017)Christiano, Leike, Brown, Martic, Legg, and
  Amodei]{christiano2017deep}
Paul~F. Christiano, Jan Leike, Tom Brown, Miljan Martic, Shane Legg, and Dario
  Amodei.
\newblock Deep reinforcement learning from human preferences.
\newblock In \emph{Advances in Neural Information Processing Systems},
  volume~30, 2017.

\bibitem[Cobbe et~al.(2021)Cobbe, Kosaraju, Bavarian, Chen, Jun, Kaiser,
  Plappert, Tworek, Hilton, Nakano, Hesse, and Schulman]{cobbe2021training}
Karl Cobbe, Vineet Kosaraju, Mohammad Bavarian, Mark Chen, Heewoo Jun, Lukasz
  Kaiser, Matthias Plappert, Jerry Tworek, Jacob Hilton, Reiichiro Nakano,
  Christopher Hesse, and John Schulman.
\newblock Training verifiers to solve math word problems.
\newblock \emph{arXiv preprint arXiv:2110.14168}, 2021.

\bibitem[Deng et~al.(2009)Deng, Dong, Socher, Li, Li, and
  Fei-Fei]{deng2009imagenet}
Jia Deng, Wei Dong, Richard Socher, Li-Jia Li, Kai Li, and Li~Fei-Fei.
\newblock {ImageNet}: A large-scale hierarchical image database.
\newblock In \emph{IEEE Conference on Computer Vision and Pattern Recognition},
  2009.

\bibitem[Gao et~al.(2023)Gao, Schulman, and Hilton]{gao2023scaling}
Leo Gao, John Schulman, and Jacob Hilton.
\newblock Scaling laws for reward model overoptimization.
\newblock In \emph{International Conference on Machine Learning}, 2023.

\bibitem[Garrabrant et~al.(2016)Garrabrant, Benson-Tilsen, Critch, Soares, and
  Taylor]{garrabrant2016logical}
Scott Garrabrant, Tsvi Benson-Tilsen, Andrew Critch, Nate Soares, and Jessica
  Taylor.
\newblock Logical induction.
\newblock \emph{arXiv preprint arXiv:1609.03543}, 2016.

\bibitem[Gneiting and Raftery(2007)]{gneiting2007strictly}
Tilmann Gneiting and Adrian~E. Raftery.
\newblock Strictly proper scoring rules, prediction, and estimation.
\newblock \emph{Journal of the American Statistical Association}, 102\penalty0
  (477):\penalty0 359--378, 2007.

\bibitem[Guo et~al.(2025)Guo, Yang, Zhang, Song, Zhang, Xu, Zhu, Ma, Wang, Bi,
  et~al.]{guo2025deepseek}
Daya Guo, Dejian Yang, Haowei Zhang, Junxiao Song, Ruoyu Zhang, Runxin Xu,
  Qihao Zhu, Shirong Ma, Peiyi Wang, Xiao Bi, et~al.
\newblock {DeepSeek-R1}: Incentivizing reasoning capability in {LLMs} via
  reinforcement learning.
\newblock \emph{arXiv preprint arXiv:2501.12948}, 2025.

\bibitem[Ha and Schmidhuber(2018)]{ha2018recurrent}
David Ha and J{\"u}rgen Schmidhuber.
\newblock Recurrent world models facilitate policy evolution.
\newblock In \emph{Advances in Neural Information Processing Systems},
  volume~31, 2018.

\bibitem[Halawi et~al.(2024)Halawi, Zhang, Yueh-Han, and
  Steinhardt]{halawi2024approaching}
Danny Halawi, Fred Zhang, Chen Yueh-Han, and Jacob Steinhardt.
\newblock Approaching human-level forecasting with language models.
\newblock \emph{arXiv preprint arXiv:2402.18563}, 2024.

\bibitem[Hendrycks et~al.(2021)Hendrycks, Burns, Kadavath, Arora, Basart, Tang,
  Song, and Steinhardt]{hendrycks2021measuring}
Dan Hendrycks, Collin Burns, Saurav Kadavath, Akul Arora, Steven Basart, Eric
  Tang, Dawn Song, and Jacob Steinhardt.
\newblock Measuring mathematical problem solving with the {MATH} dataset.
\newblock In \emph{NeurIPS Datasets and Benchmarks Track}, 2021.

\bibitem[Irving et~al.(2018)Irving, Christiano, and Amodei]{irving2018ai}
Geoffrey Irving, Paul Christiano, and Dario Amodei.
\newblock {AI} safety via debate.
\newblock \emph{arXiv preprint arXiv:1805.00899}, 2018.

\bibitem[Kaplan et~al.(2020)Kaplan, McCandlish, Henighan, Brown, Chess, Child,
  Gray, Radford, Wu, and Amodei]{kaplan2020scaling}
Jared Kaplan, Sam McCandlish, Tom Henighan, Tom~B. Brown, Benjamin Chess, Rewon
  Child, Scott Gray, Alec Radford, Jeffrey Wu, and Dario Amodei.
\newblock Scaling laws for neural language models.
\newblock \emph{arXiv preprint arXiv:2001.08361}, 2020.

\bibitem[Kirchner et~al.(2024)Kirchner, Chen, Edwards, Leike, McAleese, and
  Burda]{kirchner2024prover}
Jan~Hendrik Kirchner, Yining Chen, Harri Edwards, Jan Leike, Nat McAleese, and
  Yuri Burda.
\newblock Prover-verifier games improve legibility of {LLM} outputs.
\newblock \emph{arXiv preprint arXiv:2407.13692}, 2024.

\bibitem[LeCun(2022)]{lecun2022path}
Yann LeCun.
\newblock A path towards autonomous machine intelligence.
\newblock OpenReview preprint, 2022.

\bibitem[Leike et~al.(2018)Leike, Krueger, Everitt, Martic, Maini, and
  Legg]{leike2018scalable}
Jan Leike, David Krueger, Tom Everitt, Miljan Martic, Vishal Maini, and Shane
  Legg.
\newblock Scalable agent alignment via reward modeling: a research direction.
\newblock \emph{arXiv preprint arXiv:1811.07871}, 2018.

\bibitem[Lightman et~al.(2024)Lightman, Kosaraju, Burda, Edwards, Baker, Lee,
  Leike, Schulman, Sutskever, and Cobbe]{lightman2024lets}
Hunter Lightman, Vineet Kosaraju, Yuri Burda, Harrison Edwards, Bowen Baker,
  Teddy Lee, Jan Leike, John Schulman, Ilya Sutskever, and Karl Cobbe.
\newblock Let's verify step by step.
\newblock In \emph{International Conference on Learning Representations}, 2024.

\bibitem[{OpenAI}(2024)]{openai2024learning}
{OpenAI}.
\newblock Learning to reason with {LLMs}.
\newblock OpenAI Blog, 2024.

\bibitem[Ouyang et~al.(2022)Ouyang, Wu, Jiang, Almeida, Wainwright, Mishkin,
  Zhang, Agarwal, Slama, Ray, et~al.]{ouyang2022training}
Long Ouyang, Jeffrey Wu, Xu~Jiang, Diogo Almeida, Carroll Wainwright, Pamela
  Mishkin, Chong Zhang, Sandhini Agarwal, Katarina Slama, Alex Ray, et~al.
\newblock Training language models to follow instructions with human feedback.
\newblock In \emph{Advances in Neural Information Processing Systems},
  volume~35, 2022.

\bibitem[Rumelhart et~al.(1986)Rumelhart, Hinton, and
  Williams]{rumelhart1986learning}
David~E. Rumelhart, Geoffrey~E. Hinton, and Ronald~J. Williams.
\newblock Learning representations by back-propagating errors.
\newblock \emph{Nature}, 323\penalty0 (6088):\penalty0 533--536, 1986.

\bibitem[Savage(1971)]{savage1971elicitation}
Leonard~J. Savage.
\newblock Elicitation of personal probabilities and expectations.
\newblock \emph{Journal of the American Statistical Association}, 66\penalty0
  (336):\penalty0 783--801, 1971.

\bibitem[Silver et~al.(2017)Silver, Schrittwieser, Simonyan, Antonoglou, Huang,
  Guez, Hubert, Baker, Lai, Bolton, et~al.]{silver2017mastering}
David Silver, Julian Schrittwieser, Karen Simonyan, Ioannis Antonoglou, Aja
  Huang, Arthur Guez, Thomas Hubert, Lucas Baker, Matthew Lai, Adrian Bolton,
  et~al.
\newblock Mastering the game of {Go} without human knowledge.
\newblock \emph{Nature}, 550\penalty0 (7676):\penalty0 354--359, 2017.

\bibitem[Snell et~al.(2024)Snell, Lee, Xu, and Kumar]{snell2024scaling}
Charlie Snell, Jaehoon Lee, Kelvin Xu, and Aviral Kumar.
\newblock Scaling {LLM} test-time compute optimally can be more effective than
  scaling model parameters.
\newblock \emph{arXiv preprint arXiv:2408.03314}, 2024.

\bibitem[Strathern(1997)]{strathern1997improving}
Marilyn Strathern.
\newblock `improving ratings': audit in the {B}ritish university system.
\newblock \emph{European Review}, 5\penalty0 (3):\penalty0 305--321, 1997.

\bibitem[Sutton(2019)]{sutton2019bitter}
Richard~S. Sutton.
\newblock The bitter lesson.
\newblock Incomplete Ideas (blog), 2019.

\bibitem[Trinh et~al.(2024)Trinh, Wu, Le, He, and Luong]{trinh2024solving}
Trieu~H. Trinh, Yuhuai Wu, Quoc~V. Le, He~He, and Thang Luong.
\newblock Solving olympiad geometry without human demonstrations.
\newblock \emph{Nature}, 625\penalty0 (7995):\penalty0 476--482, 2024.

\bibitem[Uesato et~al.(2022)Uesato, Kushman, Kumar, Song, Siegel, Wang,
  Creswell, Irving, and Higgins]{uesato2022solving}
Jonathan Uesato, Nate Kushman, Ramana Kumar, Francis Song, Noah Siegel, Lisa
  Wang, Antonia Creswell, Geoffrey Irving, and Irina Higgins.
\newblock Solving math word problems with process- and outcome-based feedback.
\newblock \emph{arXiv preprint arXiv:2211.14275}, 2022.

\bibitem[Vaswani et~al.(2017)Vaswani, Shazeer, Parmar, Uszkoreit, Jones, Gomez,
  Kaiser, and Polosukhin]{vaswani2017attention}
Ashish Vaswani, Noam Shazeer, Niki Parmar, Jakob Uszkoreit, Llion Jones,
  Aidan~N. Gomez, Lukasz Kaiser, and Illia Polosukhin.
\newblock Attention is all you need.
\newblock In \emph{Advances in Neural Information Processing Systems},
  volume~30, 2017.

\bibitem[Wei et~al.(2022)Wei, Wang, Schuurmans, Bosma, Ichter, Xia, Chi, Le,
  and Zhou]{wei2022chain}
Jason Wei, Xuezhi Wang, Dale Schuurmans, Maarten Bosma, Brian Ichter, Fei Xia,
  Ed~Chi, Quoc~V. Le, and Denny Zhou.
\newblock Chain-of-thought prompting elicits reasoning in large language
  models.
\newblock In \emph{Advances in Neural Information Processing Systems},
  volume~35, 2022.

\bibitem[Wu et~al.(2022)Wu, Jiang, Li, Rabe, Staats, Jamnik, and
  Szegedy]{wu2022autoformalization}
Yuhuai Wu, Albert~Qiaochu Jiang, Wenda Li, Markus~N. Rabe, Charles Staats,
  Mateja Jamnik, and Christian Szegedy.
\newblock Autoformalization with large language models.
\newblock In \emph{Advances in Neural Information Processing Systems},
  volume~35, 2022.

\bibitem[Zelikman et~al.(2022)Zelikman, Wu, Mu, and Goodman]{zelikman2022star}
Eric Zelikman, Yuhuai Wu, Jesse Mu, and Noah Goodman.
\newblock {STaR}: Bootstrapping reasoning with reasoning.
\newblock In \emph{Advances in Neural Information Processing Systems},
  volume~35, 2022.

\end{thebibliography}

\appendix

\section{Experimental Details}
\label{app:details}

\bfpara{The DSL world.} Programs are strings of length 1--6 over six tokens acting on an integer register initialized to the input $x$: \texttt{A} ($v{+}1$), \texttt{S} ($v{-}1$), \texttt{D} ($v{\times}2$), \texttt{T} ($v{\times}3$), \texttt{N} ($-v$), \texttt{Q} ($v^2$), with executions aborted if $|v|>10^6$. A task is drawn by sampling a hidden target program of length 4 whose outputs on the evaluation inputs $x\in\{-3,\dots,4\}$ are all defined and take more than two distinct values. Gold reward of a candidate is the fraction of the eight input--output pairs it matches under execution. Random programs average gold reward 0.079; multiple functionally equivalent solutions typically exist.

\bfpara{Verifiers.} The learned verifier is ridge regression ($\alpha{=}1$) over 43 features: 6 token counts, 36 bigram counts, and length---a deliberately shallow judge evaluating the surface of a program rather than executing it. Experiment 6 uses logistic regression ($C{=}1$, predictions clipped to $[0.001,0.999]$; a clipped constant predictor when training labels are degenerate).

\bfpara{Experiment 1.} 20,000 trials per $(\rho,N)$ point; $N\in\{2^0,\dots,2^{12}\}$; $\rho\in\{0.2,0.5,0.8,1.0\}$. Theory lines use Monte-Carlo $\E[\max_N]$ ($2\times10^5$ trials), not the asymptotic formula.

\bfpara{Experiment 2.} 60 tasks; learned verifier trained on 1,000 labeled random programs per task; candidate pools of 2,048 fresh programs; $N\in\{2^0,\dots,2^{11}\}$ subsampled without replacement per point.

\bfpara{Experiment 3.} 60 tasks, 10 rounds; pool 512, settle top 32 by proxy per round; anchored condition appends settled labels and refits, frozen does not; both start from 300 base labels.

\bfpara{Experiment 4.} 40 tasks, 250 steps of steepest-ascent search: 16 mutations (substitute/insert/delete a token) per step, moving to the best-scoring candidate when it does not decrease the score. Anchored condition settles up to 25 recently visited programs every 25 steps (250 reality queries per run) and refits. The sound condition scores candidates by execution directly.

\bfpara{Experiment 5.} 40 tasks; base verifier trained on 500 random labels; budgets $S\in\{0,50,150,500,1500\}$. On-policy acquisition alternates pools of 256, settling the top 16 by the current verifier and refitting; the random condition adds $S$ labels of random programs at once. Soundness AUC is the mean of $\Snd@N$ over $N\in\{2^0,\dots,2^{11}\}$ on a fresh pool of 2,048.

\bfpara{Experiment 6.} 60 tasks, 12 rounds; claim event ``gold reward $\ge0.25$''; pools of 256 per round, mixed acquisition (16 top-probability + 16 random), 32 claims settled per round (1,920 per round pooled across tasks); 200 initial random labels per task. Brier decomposition uses 10 equal-width probability bins; reliability and resolution are computed on each round's settled claims \emph{before} that round's refit, so every point is out-of-sample.

\bfpara{Compute and seeds.} All experiments complete in under ten minutes total on a single CPU. Seeds: 0--4 (suite 1), 10--12 (suite 2), fixed in the scripts.

\section{Scaled Replication Details}
\label{app:scaled}

\bfpara{RichDSL.} Programs of length 1--10 over ten tokens acting on an integer register $v$ initialized to input $x$: \texttt{A} ($v{+}1$), \texttt{S} ($v{-}1$), \texttt{D} ($v{\times}2$), \texttt{T} ($v{\times}3$), \texttt{H} ($v\ \mathrm{div}\ 2$, truncating), \texttt{N} ($-v$), \texttt{Q} ($v^2$), \texttt{P} ($v{+}x$), \texttt{M} ($v \bmod 7$), \texttt{G} ($\max(v,0)$). Values with $|v|>10^9$ become invalid (scored as mismatch). Inputs $x\in\{-8,\dots,7\}$; targets are length-6 programs valid on all inputs with more than three distinct outputs. Random programs average gold reward 0.018. Statistical unit is the task; tests are paired two-sided Wilcoxon signed-rank unless a direction was pre-registered; intervals are 95\% bootstrap CIs.

\bfpara{Verifiers.} Weak: ridge ($\alpha{=}1$) on 111 features (unigrams, bigrams, length). Strong: gradient-boosted trees (\texttt{HistGradientBoostingRegressor}, 100 iterations; 60 in the H3 loop) on 359 features (adding positional token indicators, 128 hashed trigrams, first/last one-hots). Both trained on 2,000 labeled random programs per task (600 base labels in the loop experiments).

\bfpara{H1.} Gaussian simulation: 100,000 trials per point, $N\in\{2^0,\dots,2^{14}\}$, $\rho\in\{0.1,0.3,0.5,0.7,0.9,1.0\}$; the penalty sweep uses 20,000 trials to $N{=}2^{17}$. Transfer (H1c): per-task Pearson $\rho(\text{proxy},\text{gold})$ on fresh pools of 4,096 vs.\ realized $\Snd@4096$, 120 tasks.

\bfpara{H2.} 120 tasks; pools of 4,096; $N\in\{2^0,\dots,2^{12}\}$, eight subsample repetitions per point.

\bfpara{H3.} 12 rounds; pool 1,024/round; settle the top 32 by proxy; anchored refits on all settled labels; 120 tasks (weak), 60 (strong).

\bfpara{H4.} 300 steps, 24 mutations (substitute/insert/delete) per step, 100 tasks; anchored settlement of the 25 most recent unique visited programs every 25 steps (300 reality queries/run). Steepest ascent moves to the best mutation when not worse; the evolutionary adversary keeps the top 16 of population-plus-children; the settlement-aware adversary detects refits (score drop without moving) and restarts at the proxy-best of 64 random probes filtered to the 32 most edit-distant from all settled programs. Its population (evolutionary) or the anchored model's scores are recomputed after refits; the steepest climber's stale incumbent score is the registered-failure artifact discussed in \S\ref{sec:replication}.

\bfpara{H5.} Budgets $S\in\{0,100,300,1000,3000\}$ in batches of 20; on-policy settles the top 20 by the current verifier from pools of 512; uncertainty sampling scores pool variance under a 3-member bootstrap-ridge committee; Soundness AUC is the mean of $\Snd@N$ over $N\in\{2^0,\dots,2^{12}\}$ on fresh pools of 4,096; 80 tasks.

\bfpara{H6.} Claim event ``gold $\ge0.125$'' (base rate $\approx0.19$; a pre-run, documented amendment from the minimal suite's 0.25, whose base rate $\approx0.05$ in RichDSL is too degenerate for 10-bin decomposition); 15 rounds; pools of 256; 16 top-probability + 16 random claims settled per round; decomposition computed out-of-sample before each refit; 100 tasks.

\bfpara{Compute and seeds.} $\sim\!75$ minutes on a single CPU. Tasks are independently seeded by index (bases 20,000 / 50,000 / 60,000 / 70,000 / 90,000 / 110,000), so chunked and monolithic invocations produce identical results.

\end{document}